\documentclass[moor,nonblindrev]{informs3opt}
\usepackage{color,soul}
\usepackage{palatino,url}%
\usepackage{verbatim, amssymb}
\usepackage{amsmath,mathrsfs,thmtools}
\usepackage{graphics,amsopn,amstext,amsfonts,color}
\usepackage{bm}
\usepackage{setspace}
\usepackage{subfigure}
\usepackage{graphicx}
\usepackage{epstopdf}
\usepackage{natbib}

\usepackage{multirow}
\usepackage{booktabs}
\usepackage{paralist}
\usepackage{colortbl}
\usepackage{titlesec}
\usepackage[titletoc,toc,title]{appendix}

\usepackage{xcolor}
\usepackage{cases}
\usepackage{empheq}
\usepackage[hidelinks,bookmarksnumbered=true]{hyperref}
\usepackage{graphics,latexsym,amsfonts,verbatim,amsmath}

\usepackage{algorithm}
\usepackage{algpseudocode}
\usepackage{todonotes}

\def\E{{\mathbb E}}

\def\Pr{{\mathbb{P}}}

\def\Re{\mathbb{R}}
\def\Qe{\mathbb{Q}}
\def\I{\mathbb{I}}

\def\hat{\widehat}

\def \A{\mathcal{A}}
\def \R{\mathcal{R}}

\def \Ze{{\mathbb{Z}}}

\def\N{{\mathcal N}}

\def\Re{{\mathbb R}}

\newcommand{\exclude}[1]{}

\newcommand{\e}{\mathbf{e}}

\renewcommand{\S}{\mathcal{\tilde S}}

\DeclareMathOperator{\Supp}{\mathrm{Supp}}

\DeclareMathOperator{\diag}{diag}

\usepackage{authblk}
\algdef{SE}[DOWHILE]{Do}{doWhile}{\algorithmicdo}[1]{\algorithmicwhile\ #1}%
\algnewcommand{\Or}{\textbf{or}}
\algnewcommand{\And}{\textbf{and}}
\usepackage{thmtools}
\usepackage{thm-restate}

\usepackage{cleveref}

\declaretheorem[name=Theorem]{theorem}
\declaretheorem[name=Proposition]{proposition}
\declaretheorem[name=Lemma]{lemma}
\declaretheorem[name=Claim]{claim}
\declaretheorem[name=Corollary]{corollary}

\declaretheorem[name=Observation]{observation}

\declaretheorem[name=Definition]{definition}

\newcommand*{\QEDA}{\hfill\ensuremath{\square}}
\newcommand*{\QEDB}{\hfill\ensuremath{\diamond}}

\usepackage{natbib}
\NatBibNumeric
\def\bibfont{\small}%
\bibpunct[, ]{[}{]}{,}{n}{}{,}%

\allowdisplaybreaks
\begin{document}

\RUNAUTHOR{Mohit Singh, Weijun Xie}
\RUNTITLE{Approximation Algorithms for $D$-optimal Design
}

\TITLE{Approximation Algorithms for $D$-optimal Design}

\ARTICLEAUTHORS{%
	
	\AUTHOR{Mohit Singh}
	\AFF{School of Industrial and Systems Engineering, Georgia Institute of Technology, Atlanta, GA 30332, \EMAIL{mohitsinghr@gmail.com.}}
	
		\AUTHOR{Weijun Xie}
	\AFF{Department of Industrial and Systems Engineering, Virginia Tech, Blacksburg, VA 24061, \EMAIL{wxie@vt.edu.}}
}

\ABSTRACT{
Experimental design is a classical statistics problem and its aim is to estimate an unknown $m$-dimensional vector $\bm \beta$ from linear measurements where a Gaussian noise is introduced in each measurement. For the combinatorial experimental design problem, the goal is to pick $k$ out of the given $n$ experiments so as to make the most accurate estimate of the unknown parameters, denoted as $\hat{\bm \beta}$. 
In this paper, we will study one of the most robust measures of error estimation - $D$-optimality criterion, which corresponds to minimizing the volume of the confidence ellipsoid for the estimation error $\bm \beta-\hat{\bm \beta}$. The problem gives rise to two natural variants depending on whether repetitions of experiments are allowed or not.  We first propose an approximation algorithm with a $\frac1e$-approximation for the $D$-optimal design problem with and without repetitions, giving the first constant factor approximation for the problem. We then analyze another sampling approximation algorithm and prove that it is $(1-\epsilon)$-approximation if $k\geq \frac{4m}{\epsilon}+\frac{12}{\epsilon^2}\log(\frac{1}{\epsilon})$ for any $\epsilon \in (0,1)$. Finally, for $D$-optimal design with repetitions, we study a different algorithm proposed by literature and show that it can improve this asymptotic approximation ratio.
}

 \HISTORY{}

\maketitle

\section{Introduction}
%
%
%

Experimental design is a classical problem in statistics~\cite{atkinson2007optimum,federer1955experimental,fedorov1972theory,kirk1982experimental,pukelsheim2006optimal} and recently has also been applied to machine learning~\cite{Allen-Zhu17nearoptimal,wang2016computationally}. In an experimental design problem, its aim is to estimate an $m$-dimensional vector $\bm{\beta}\in \Re^m$ from $n$ linear measurements of the form $y_i=\bm{a}_i^T \bm{\beta} +\tilde{\epsilon}_i$ for each $i \in [n]:=\{1,2,\ldots,n\}$, where vector $\bm{a}_i\in \Re^m$ characterizes the $i^{th}$ experiment and $\{\tilde{\epsilon}_i\}_{i\in [n]}$ are i.i.d. independent Gaussian random variables with zero mean and variance $\sigma^2$ (i.e., $\tilde{\epsilon}_i\sim\N(0,\sigma^2)$ for all $i\in [n]$). Due to limited resources, it might be quite expensive to conduct all the $n$ experiments. Therefore, as a compromise, in the combinatorial experimental design problem, we are given an integer $k \in [m,n]$, and our goal is to pick $k$ out of the $n$ experiments so as to make the \emph{most accurate estimate} of the parameters, denoted as $\hat{\bm \beta}$. Suppose that a size $k$ subset of experiments $S\subseteq [n]$ is chosen, then the most likelihood estimation of $\bm{\beta}$ (cf., \cite{joshi2009sensor}) is given by $$\hat{\bm \beta}=\left(\sum_{i\in S} \bm{a}_i\bm{a}_i^{\top}\right)^{-1} \sum_{i\in S} y_i \bm{a}_i.$$
There are many criteria on how to choose the best estimation among all the possible size $k$ subsets of $n$ experiments (see \cite{atkinson2007optimum} for a review).
One of the most robust measures of error estimation is known as $D$-optimality criterion, where the goal is to choose the best size $k$ subset $S$ to maximize $\left[\det\left(\sum_{i\in S} \bm{a}_i\bm{a}_i^{\top}\right)\right]^{\frac{1}{m}}$, i.e., the following combinatorial optimization problem:
\begin{align}
\max_{S}\left\{f(S):=\left[\det\left(\sum_{i\in S} \bm{a}_i\bm{a}_i^{\top}\right)\right]^{\frac{1}{m}}:\Supp(S)\subseteq [n],|S|=k \right\}, \label{eq_d_opt_combinatorial}
\end{align}
where $\Supp(S)$ denotes the support of set $S$ and $|S|$ denotes the cardinality of set $S$. Note that optimization model \eqref{eq_d_opt_combinatorial} corresponds to minimizing the volume of the confidence ellipsoid for the estimation error $\bm \beta-\hat{\bm \beta}$. Equivalently, the objective function of \eqref{eq_d_opt_combinatorial} can be chosen as determinant itself (i.e., $\det\left(\sum_{i\in S} \bm{a}_i\bm{a}_i^{\top}\right)$), or log-determinant $\log\det\left(\sum_{i\in S} \bm{a}_i\bm{a}_i^{\top}\right))$ \cite{joshi2009sensor}. However, both objective functions are problematic with the following reasons: (1) the determinant function is nonconvex and has numerical issue especially when $k,m,n$ are large; and (2) albeit log-determinant function is concave, it can also be numerically unstable, in particular when the determinant is close to zero. Therefore, in this paper, we follow the work in \cite{sagnol2013computing} and consider $m$th root of determinant function, which is concave and numerically stable.

In the problem description, the same experiment may or may not be allowed to choose multiple times. We refer the problem as \emph{$D$-optimal design with repetitions} if we are allowed to pick an experiment more than once and \emph{$D$-optimal design without repetitions} otherwise. Correspondingly, for $D$-optimal design with repetitions and without repetitions, in \eqref{eq_d_opt_combinatorial}, the subset $S$ denotes a \emph{multi-set}, where elements of $S$ can be duplicated, and a \emph{conventional} set, where elements of $S$ must be different from each other, respectively.
The former problem has been studied very extensively in statistics \cite{kirk1982experimental,pukelsheim2006optimal,sagnol2013computing}. The latter problem has also been studied as the sensor selection problem~\cite{joshi2009sensor}, where the goal is to find the best subset of sensor locations to obtain the most accurate estimate of unknown parameters. It is easy to see that $D$-optimal design with repetitions is a special case of the $D$-optimal design problem without repetitions. To do so,  for the $D$-optimal design with repetitions, we can create $k$ copies of each vector, which reduces to the $D$-optimal design without repetitions with $nk$ vectors.

The remainder of the paper is organized as follows. Section~\ref{sec_setup} details the problem setting, reviews related literature, and summarizes our contributions. Section~\ref{sec:1/e} develops and analyzes a randomized algorithm, its approximation results and deterministic implementation. Section~\ref{proof-theorem:main2} proposes another randomized algorithm and its deterministic implementation as well as asymptotic behavior. Section~\ref{proof-theorem:main3} analyzes a randomized algorithm for $D$-optimal design with repetitions proposed by literature, and investigates its deterministic implementation as well as approximation ratios. Finally, Section~\ref{sec_con} concludes the paper.\\

\noindent\textit{Notations:} The following notation is used throughout the paper. We use bold-letters (e.g., $\bm{x},\bm{A}$) to denote vectors or matrices, and use corresponding non-bold letters to denote their components. We let $\e$ be the all-ones vector. 
We let $\Re_+,\Qe_+,\Ze_+$ denote the sets of positive real numbers, rational numbers and integers, respectively. Given a positive integer $N$ and a set $Q$, we let $[N]:=\{1,\ldots,N\}$, $N!=\prod_{i\in [N]}i$, $|Q|$ denote its cardinality, $\Supp(Q)$ denotes the support of $Q$ and ${Q\choose N}$ denote all the possible subsets of $Q$ whose cardinality equals to $N$. Given a multi-set $Q$, for each $i\in \Supp(Q)\subseteq [n]$, we let function $M_{Q}(i)$ denote the number of occurrences of $i$ in $Q$, and for any $i\in [n]$, $Q(i)$ denotes its $i$th component. Given a matrix $\bm A$ and two sets $R,T$, we let $\det(\bm A)$ denote its determinant if $\bm A$ is a square matrix, let $\bm A_{R,T}$ denote a submatrix of $\bm A$ with rows and columns from sets $R,T$, let $\bm A_i$ denote $i$th column of matrix $\bm A$ and let $\bm A_R$ denote submatrix of $\bm A$ with columns from set $R$. For any positive integer $r$, we let $\bm{I}_r$ denote $r\times r$ identity matrix. We use $\S$ to denote a random set. For notational convenience, give a set $S\in [n]$ and vector $\bm{x}\in \Re_+^n$, we define $f(S)=\left[\det\left(\sum_{i\in S} \bm{a}_i\bm{a}_i^{\top}\right)\right]^{\frac{1}{m}}$ and $f(x)=\left[\det\left(\sum_{i\in [n]} x_i\bm{a}_i\bm{a}_i^{\top}\right)\right]^{\frac{1}{m}}$.
Additional notation will be introduced as needed.

\section{Model Formulation, Related Literature and Contributions}\label{sec_setup}

\subsection{Model Formulation}
To formulate $D$-optimal design problem \eqref{eq_d_opt_combinatorial} as an equivalent mathematical program, we first let set $B$ denote the set of all nonnegative numbers (i.e., $B=\Re_+$) if repetitions are allowed, otherwise, $B=[0,1]$. Next, we introduce integer decision variable $x_i\in B\cap \Ze_+$ to represent how many times $i$th experiment will be chosen for each $i\in [n]$. With the notation introduced above, the $D$-optimal design problem can be formulated as a mixed integer convex program below:
\begin{align}\label{opt_sensor}
w^*:=\max_{\bm x,w}\left\{w:w\leq f(x)=\left[\det\left(\sum_{i\in [n]}x_i\bm{a}_i\bm{a}_i^\top\right)\right]^{\frac{1}{m}},\sum_{i\in [n]}x_i=k,\bm{x}\in B^n\cap \Ze_+^n\right\},
\end{align}
where for notational convenience, we let $f(\bm{x})$ to denote the objective function. Note that if $B=\Re_+$, \eqref{opt_sensor} is equivalent to $D$-optimal design with repetitions and if $B=[0,1]$, then \eqref{opt_sensor} corresponds to $D$-optimal design without repetitions.

It can be easily shown that $f(\bm{x})=\left[\det\left(\sum_{i\in [n]}x_i\bm{a}_i\bm{a}_i^\top\right)\right]^{\frac{1}{m}}$ is concave in $\bm{x}$ (cf., \cite{ben2001lectures}). Therefore, a straightforward convex relaxation of \eqref{opt_sensor} is to relax the binary vector $\bm x$ to continuous, i.e., $\bm x\in [0,1]^n$, which is formulated as below:
\begin{align}\label{opt_sensor_convex}
\hat{w}:=\max_{\bm x,w}\left\{w:w\leq \left[\det\left(\sum_{i\in [n]}x_i\bm{a}_i\bm{a}_i^\top\right)\right]^{\frac{1}{m}},\sum_{i\in [n]}x_i=k,\bm{x}\in B^n\right\}.
\end{align}
Note that \eqref{opt_sensor_convex} is a tractable convex program (cf., \cite{joshi2009sensor}), thus is efficiently solvable. Recently, in \cite{sagnol2013computing}, the authors proposed an alternative second order conic program (SOCP) formulation for \eqref{opt_sensor_convex}, which can be solved by a more effective interior point method or even off-the-shelf solvers (e.g., CPLEX, Gurobi, MOSEK). {We remark that according to \cite{ben2001lectures}, the time complexity of solving \eqref{opt_sensor_convex} is $O(n^5)$.}


\subsection{Related Literature}

As remarked earlier, experimental design is a classical area in Statistics. We refer the reader to Pukelsheim \cite{pukelsheim2006optimal}, Chapter 9 on details about $D$-optimality criterion as well as other related $(A,E,T)$-criteria. The combinatorial version, where the number of each experiment needs to be chosen is an integer as in \eqref{opt_sensor}, is also called exact experimental design. It turns out that the $D$-optimality criterion is proven to be NP-hard~\cite{welch1982algorithmic}. 
Thus, there has been plenty of works on heuristic methods, including local search and its variants, to obtain good solutions~\cite{fedorov1972theory,joshi2009sensor}.

From an approximation algorithm viewpoint, $D$-optimal design has received a lot of attention recently. For example, Bouhou et al.~\cite{bouhtou2010submodularity} gave a $\left(\frac{n}{k}\right)^{\frac{1}{m}}$-approximation algorithm, and Wang et al.~\cite{wang2016computationally} building on \cite{avron2013efficient} gave a $(1+\epsilon)$-approximation if $k\geq \frac{m^2}{\epsilon}$. Recently, Allen-Zhu et al~\cite{Allen-Zhu17nearoptimal} realized the connection between this problem and matrix sparsification \cite{batson2012twice,spielman2011graph} and used regret minimization methods~\cite{allen2015spectral} to obtain $O(1)$-approximation algorithm if $k\geq 2m$ and $(1+\epsilon)$-approximation when $k\geq O\left(\frac{m}{\epsilon^2}\right)$. We also remark that their results are general and also are applicable to other optimality criteria.

Another closely related problem is the largest $j$-simplex problem, whose problem description is as follows:
\begin{quote}
(largest $j$-simplex problem) Given a set of $n$ vectors $\bm{a}_1,\ldots, \bm{a}_n \in \Re^m$ and integer $k\leq m$, pick a set of $S$ of $k$ vectors to maximize the $k^{th}$-root of the pseudo-determinant of $\bm{X}=\sum_{i\in S} \bm{a}_i\bm{a}_i^{\top}$, i.e., the geometric mean of the non-zero eigenvalues of $X$.
\end{quote} The problem has also received much attention recently~\cite{khachiyan1996rounding,nikolov2015randomized,summa2015largest} and Nikolov\cite{nikolov2015randomized} gave a $\frac{1}{e}$-approximation algorithm. Observe that the special case of $k=m$ of this problem coincides with the special case of $k=m$ for the $D$-optimal design problem. Indeed, Nikolov's algorithm, while applicable, results in a $e^{-\frac{k}{m}}$-approximation for the $D$-optimal design problem. Recently, matroid constrained versions of the largest $j$-simplex problem have also been studied~\cite{anari2017generalization,nikolov2016maximizing,straszak2016real}.

The $D$-optimality criterion is also closely related to constrained submodular maximization~\cite{nemhauser1978analysis}, a classical combinatorial problem, for which there has been much progress recently \cite{calinescu2011maximizing,krause2014submodular}. Indeed, the set function $m\log f(S):=\log \det\left(\sum_{i\in S} \bm{a}_i\bm{a}_i^{\top}\right)$ is known to be submodular \cite{shamaiah2010greedy}. Unfortunately, this submodular function is not necessarily non-negative, a prerequisite for all the results on constrained submodular maximization and thus these results are not directly applicable. We also remark that for a multiplicative guarantee for the $\det$ objective, we would aim for an additive guarantee for $\log \det$ objective.

\subsection{Contributions}
In this paper, we make several contributions to the approximations of $D$-optimal design problems both with and without repetitions. Our approximation algorithms are randomized, where we sample $k$ experiments out of $n$ with given marginals. This type of sampling procedure has been studied intensively in the approximation algorithms and many different schemes have been proposed~\cite{brewer2013sampling}, where most of them exhibit a negative correlation of various degrees~\cite{branden2012negative}. In this work, we will study sampling schemes which exhibit approximate positive correlation and prove the approximation ratios of these algorithms. 

All the proposed approximations come with approximation guarantees. Given an {\em approximation ratio} $\gamma\in(0,1]$, a $\gamma$-approximation algorithm for $D$-optimal design returns a solution $\bar{\bm{x}} \in B^n\cap \Ze_+^n$, such that $ \sum_{i\in [n]}\bar{x}_i=k$ and $f(\bar{\bm{x}})\geq \gamma w^*$,
i.e., the solution is feasible and has an objective value at least $\gamma$ times the optimal value. The approximation ratios of our randomized algorithms only hold in the sense of expectation. To improve them, we further propose polynomial-time deterministic implementations for all the randomized algorithms with iterated conditional expectation method, which also achieve the same approximation guarantees.
Following is a summary of our contributions.
\begin{enumerate}
	\item We develop a $\frac{1}{e}$-approximation algorithm and its polynomial-time deterministic implementation, giving the first constant factor approximation for $D$-optimal design problem for both with and without repetitions. Previously, constant factor approximations were known only for a restricted range of parameters~\cite{Allen-Zhu17nearoptimal,avron2013efficient,wang2016computationally} (see related work for details).
	\item We study a different sampling algorithm and its polynomial-time deterministic implementation, showing that its solution is $(1-\epsilon)$ optimal if $k\geq \frac{4m}{\epsilon}+\frac{12}{\epsilon^2}\log(\frac{1}{\epsilon})$ for any given $\epsilon \in (0,1)$. This results substantially improve the previous work \cite{Allen-Zhu17nearoptimal,wang2016computationally}.
	\item For $D$-optimal design with repetitions, we investigate a simple randomized algorithm similar to that in Nikolov~\cite{nikolov2015randomized}, study its polynomial-time deterministic implementation and provide a significant different analysis of the approximation guarantee. We show that the proposed algorithm yields $(1-\epsilon)$-approximation for the $D$-optimal design problem with repetitions when $k\geq \frac{m-1}{\epsilon}$.
\end{enumerate}
Note that the preliminary version of the paper has appeared in ACM-SIAM Symposium on Discrete Algorithms (SODA18) \cite{singh2018approximate}. Compared to \cite{singh2018approximate}, the current paper has the following major improvement: (1) for the constant-factor approximation algorithm presented in Section~\ref{sec:1/e}, we simplify its polynomial-time deterministic implementation and improve its analysis, (2) for the asymptotically optimal algorithm in \Cref{proof-theorem:main2}, we simplify its sampling procedure and derive its polynomial-time deterministic implementation, and (3) for the approximation algorithm of $D$-optimal design with repetitions, we improve its approximation ratio analysis and propose its polynomial-time  deterministic implementation.

\section{Approximation Algorithm for $D$-optimal Design Problem}\label{sec:1/e}
In this section, we will propose a sampling procedure and prove its approximation ratio. We also develop an efficient way to implement this algorithm and finally we will show its polynomial-time deterministic implementation with the same performance guarantees.

First of all, we note that for $D$-optimal design problem with repetitions, i.e., $B=\Re_+$ in \eqref{opt_sensor}, it can be equivalently reformulated as $D$-optimal design problem without repetitions by creating $k$ copies of vectors $\{\bm a_i\}_{i\in [n]}$. Therefore, the approximation for $D$-optimal design problem without repetitions directly apply to that with repetitions. Hence, in this and next sections, we will only focus on $D$-optimal design problem without repetitions, i.e., in \eqref{opt_sensor} and \eqref{opt_sensor_convex}, we only consider $B=[0,1]$.


\subsection{Sampling Algorithm and Its Efficient Implementation}\label{sec_efficien_Implement}
In this subsection, we will introduce a sampling procedure and explain its efficient implementation.

In this sampling algorithm, we first suppose $(\hat{\bm x},\hat w)$ to be an optimal solution to the convex relaxation \eqref{opt_sensor_convex}, where $\hat{\bm x} \in [0,1]^n$ with $\sum_{i\in [n]}\hat{x} _i=k$ and $\hat{w}=f(\hat{\bm x})$. Then we randomly choose a size-$k$ subset $\S$ according to the following probability:
\begin{align}
\Pr[\S=S]=\frac{\prod_{j\in S}\hat{x}_j}{
	\sum_{\bar{S}\in {[n]\choose k}}\prod_{i\in \bar{S}}\hat{x}_i}\label{eq:sampling1}
\end{align}
for every $S\in {[n]\choose k}$, { where $[n]\choose k$ denotes all the possible size-$k$ subsets of $[n]$.}

Note that the sampling procedure in \eqref{eq:sampling1} is not an efficient implementable description since there are $n\choose k$ candidate subsets to be sampled from. Therefore, we propose an efficient way (i.e., Algorithm~\ref{alg_rand_round}) to obtain a size-$k$ random subset $\S$ with probability distribution in \eqref{eq:sampling1}. We first observe a useful way to compute the probabilities in Algorithm~\ref{alg_rand_round} efficiently.
\begin{observation}\label{lem_key_lem3}
	Suppose $\bm x\in \Re^t$ and integer $0\leq r\leq t$, then $
	\sum_{S\in {[t]\choose r}}\prod_{i\in S} x_i$
	is the coefficient of $y^r$ of the polynomial $\prod_{i\in [t]}(1+x_iy)$.
\end{observation}
{In fact, it has been shown that the product of two polynomials with degree at most $t$ can be done in $O(t\log t)$ amount of time by the Fast Fourier Transform (FFT) \cite{cooley1965algorithm}. Thus, by divide-and-conquer approach, it takes $O(t\log^2t)$ time to expand the polynomial $\prod_{i\in [t]}(1+x_iy)$, i.e., it takes $O(t\log^2t)$ time to compute the coefficient of $y^r$ of the polynomial $\prod_{i\in [t]}(1+x_iy)$.
}

Note that we need to compute the probability $\Pr[\S=S]$ in \eqref{eq:sampling1} efficiently. The main idea of the efficient implementation is to sample elements one by one based on conditioning on the chosen elements and unchosen elements and update the conditional probability distribution sequentially. Indeed, suppose that we are given a subset of chosen elements $S$ with $|S|<k$ and a subset of unchosen elements $T$ with $|T|<n-k$ which can be empty, then probability that the $j$th experiment with $j\in [n]\setminus (S\cup T)$ will be chosen is equal to
\[\Pr[\text{$j$ will be chosen}|S,T]=\frac{\hat{x}_j\left(\sum_{\bar{S}\in {[n]\setminus (S\cup T)\choose k-1-|S|}}\prod_{\tau\in \bar{S}}\hat{x}_{\tau}\right)}{
	\left(\sum_{\bar{S}\in {[n]\setminus (S\cup T)\choose k-|S|}}\prod_{\tau\in \bar{S}}\hat{x}_{\tau}\right)}.\]
In the above formula, the denominator and numerator can be computed efficiently based on Observation~\ref{lem_key_lem3}. Thus, we flip a coin with success rate equal to the above probability, which clearly has the following two outcomes: if $j$ is chosen, then update $S:=S\cup\{j\}$; otherwise, update $T:=T\cup\{j\}$. Then, go to next iteration and repeat this procedure until $|S|=k$. By applying iterated conditional probability, the probability that $S$ is chosen equal to $\frac{\prod_{j\in S}\hat{x}_j}{\sum_{\bar{S}\in {[n]\choose k}}\prod_{j\in \bar{S}}\hat{x}_{j}}$, i.e., \eqref{eq:sampling1} holds.
The detailed implementation is shown in Algorithm~\ref{alg_rand_round}. {Note that the time complexity of Algorithm~\ref{alg_rand_round} is $O(n^2)$.}

\begin{algorithm}[ht]
	\caption{{Efficient Implementation of Sampling Procedure \eqref{eq:sampling1} with Constant Factor Approximation}}
	\label{alg_rand_round}
	\begin{algorithmic}[1]
		\State Suppose $(\hat{\bm x},\hat w)$ is an optimal solution to the convex relaxation \eqref{opt_sensor_convex} with $B=[0,1]$, where $\hat{\bm x} \in [0,1]^n$ with $\sum_{i\in [n]}\hat{x} _i=k$ and $\hat{w}=f(\hat{\bm x})$
		\State Initialize chosen set $\S=\emptyset$ and unchosen set $T= \emptyset$
		\State Two factors: $A_1=\sum_{\bar{S}\in {[n]\choose k}}\prod_{i\in \bar{S}}\hat{x}_i,A_2=0$
		\For{$j=1,\ldots,n$}
		\If{$|\S|==k$}
		\State break
		\ElsIf{$|T|=n-k$}
		\State $\S=[n]\setminus T$
		\State break
		\EndIf
		\State Let $A_2=\left(\sum_{\bar{S}\in {[n]\setminus (\S\cup T)\choose k-1-|\S|}}\prod_{\tau\in \bar{S}}\hat{x}_{\tau}\right)$
		\State Sample a $(0,1)$ uniform random variable $U$
		\If{$\hat{x}_jA_2/A_1 \geq U$}
		\State Add $j$ to set $\S$
		\State $A_1=A_2$
		\Else
		\State Add $j$ to set $T$
		\State $A_1=A_1-\hat{x}_jA_2$
		\EndIf
		\EndFor
		\State Output $\S$
	\end{algorithmic}
\end{algorithm}

\subsection{$m$-wise $\alpha$-positively Correlated Probability Distributions}

In this subsection, we will introduce the main proof idea of the approximation guarantees, which is to analyze the probability distribution \eqref{eq:sampling1} and show that it is approximately positively correlated (see, e.g., \cite{byrka2014improved}). The formal derivation will be in the next subsection.

{Recall that a set of random variables $X_1,\ldots, X_n$ are pairwise positively correlated if for each $i,j\in [n]$, we have $\E[X_i X_j]\geq \E[X_i]\cdot \E[X_j]$, which for $\{0,1\}$-valued random variables translates to $\Pr[X_i=1, X_j=1]\geq \Pr[X_i=1]\cdot \Pr[X_j=1]$. This definition aims to capture settings where random variables are more likely to take similar values than independent random variables with the same marginals. More generally, given an integer $m$, $\{0,1\}$-valued random variables $X_1,\ldots, X_n$ are called $m$-wise positively correlated random variables if $\Pr[X_i=1\;\forall i\in T]\geq \prod_{i\in T} Pr[X_i=1]$ for all $T\subseteq [n]$ where $|T|=m$. Below we provide an generalized definition of positive correlation that is crucial to our analysis.}

\begin{definition}\label{def:m-alpha} Given $\bm x\in [0,1]^n$ such that $\sum_{i\in [n]}x_i=k$ for integer $k\geq 1$, let $\mu$ be a probability distribution on subsets of $[n]$ of size $k$. Let $X_1,\ldots, X_n$ denote the indicator random variables, 
	thus $X_i=1$ if $i\in \S$ and $0$ otherwise for each $i\in [n]$ where random set $\S$ of size-$k$ is sampled from $\mu$. Then $X_1,\ldots, X_n$ are \emph{$m$-wise $\alpha$-positively correlated} for some $\alpha\in [0,1]$ if
	for each $T\subseteq [n]$ such that $|T|=m$, we have
	$$\Pr\left[X_i=1\;\forall i\in T\right]=\Pr\left[T\subseteq \S\right]\geq \alpha^{m} \prod_{i\in T} x_i.$$
\end{definition}

With a slight abuse of notation, we call the distribution $\mu$ to be \emph{$m$-wise $\alpha$-positively correlated} with respect to $\bm{x}$ if the above condition is satisfied. Observe that if $\alpha=1$, then the above definition implies that the random variables $X_1,\ldots, X_n$ are $m$-wise positively correlated. 

The following lemma shows the crucial role played by $m$-wise approximate positively correlated distributions in the design of algorithms for $D$-optimal design.
\begin{lemma}\label{lemma:reduction}
	Suppose $(\hat{\bm x},\hat w)$ is an an optimal solution to the convex relaxation \eqref{opt_sensor_convex}. Then for any $\alpha \in (0,1]$, if there exists an efficiently computable distribution that is $m$-wise $\alpha$-positively correlated with respect to $\bm{\hat x}$, then the $D$-optimal design problem has a randomized $\alpha$-approximation algorithm, i.e.,
	\[\left\{\E\left[\det\left(\sum_{i\in \S}\bm{a}_i\bm{a}_i^{\top}\right)\right]\right\}^{\frac{1}{m}}\geq \alpha w^*\]
	where random set $\S$ with size $k$ is the output of the approximation algorithm.
\end{lemma}
The proof of Lemma~\ref{lemma:reduction} relies on the polynomial formulation of matrix determinant and convex relaxation of the $D$-optimal design problem. We show that a $m$-wise $\alpha$-positively correlated distribution leads to a randomized algorithm for the $D$-optimal design problem that approximates each of the coefficients in the polynomial formulation. Note that the approximation ratio $\alpha$ in \Cref{lemma:reduction} only holds in the sense of expectation. Therefore, one might need to derandomize the algorithm to achieve the approximation ratio.

Before proving Lemma~\ref{lemma:reduction}, we would like to introduce some useful results below. The following lemmas follow from Cauchy-Binet formula~\cite{broida1989comprehensive} and use the fact that a matrix's determinant is polynomial in entries of the matrix. Interested readers can find the proofs in the appendix. 
\begin{restatable}{lemma}{emkeylem}\label{lem_key_lem}
	Suppose $\bm{a}_i\in \Re^m$ for $i\in T$ with $|T|\geq m$, then
	\begin{align}
	\det\left(\sum_{i\in T}\bm{a}_i\bm{a}_i^\top\right) =\sum_{S\in {T\choose m}}\det\left(\sum_{i\in S}\bm{a}_i\bm{a}_i^\top\right).\label{eq_identity_key_lem}
	\end{align}
\end{restatable}
\begin{proof}See Appendix~\ref{proof_lem_key_lem}.
\QEDA\end{proof}

\begin{restatable}{lemma}{emkeylemm}\label{lem_key_lem2}
	For any $\bm x \in[0,1]^n$, then
	\begin{align}
	\det\left(\sum_{i\in [n]}x_i\bm{a}_i\bm{a}_i^\top\right) =\sum_{S\in {[n]\choose m}}\prod_{i\in S}x_i
	\det\left(\sum_{i\in S}\bm{a}_i\bm{a}_i^\top\right).\label{eq_identity_key_lem2}
	\end{align}
\end{restatable}
\begin{proof}See Appendix~\ref{proof_lem_key_lem2}.
\QEDA\end{proof}

Now we are ready to prove the main Lemma~\ref{lemma:reduction}.\\
\begin{proof}(Proof of Lemma~\ref{lemma:reduction})
Note that $(\hat{\bm x},\hat w)$ is an optimal solution to \eqref{opt_sensor_convex}, and the distribution $\mu$ given by Lemma~\ref{lemma:reduction} for this $\hat{\bm x}$ which satisfies the conditions of Definition~\ref{def:m-alpha}. We now consider the randomized algorithm that samples a random set $\S$ from this distribution $\mu$. We show this randomized algorithm satisfies the guarantee claimed in the Lemma. All expectations and probabilities of events are under the probability measure $\mu$ and for simplicity, we drop it from the notation.
	
Since $\left[\det\left(\sum_{i\in [n]}\hat{x}_i\bm{a}_i\bm{a}_i^{\top}\right)\right]^{\frac{1}{m}}$ is at least as large as the optimal value to $D$-optimal design problem~\eqref{opt_sensor_convex} with $B=[0,1]$, we only need to show that
\begin{align*}
	\left\{\E\left[\det\left(\sum_{i\in \S}\bm{a}_i\bm{a}_i^{\top}\right)\right]\right\}^{\frac{1}{m}} \geq \alpha\left[\det\left(\sum_{i\in [n]}\hat{x}_i\bm{a}_i\bm{a}_i^{\top}\right)\right]^{\frac{1}{m}} ,
\end{align*}
or equivalently
	\begin{align}
	\E\left[\det\left(\sum_{i\in \S}\bm{a}_i\bm{a}_i^{\top}\right)\right] \geq \alpha^m\det\left(\sum_{i\in [n]}\hat{x}_i\bm{a}_i\bm{a}_i^{\top}\right).
	\end{align}
	This indeed holds since
	\begin{align*}
	&\E\left[\det\left(\sum_{i\in \S}\bm{a}_i\bm{a}_i^{\top}\right)\right]
	=\sum_{S\in {[n]\choose k}}\Pr\left[\S=S\right]
	\det\left(\sum_{i\in S}\bm{a}_i\bm{a}_i^\top\right)
	=\sum_{S\in {[n]\choose k}}\Pr\left[\S=S\right]
	\sum_{T\in {S\choose m}}\det\left(\sum_{i\in T}\bm{a}_i\bm{a}_i^\top\right)
	\\
	&=\sum_{T\in {[n]\choose m}}\Pr\left[T\subseteq\S\right]
	\det\left(\sum_{i\in T}\bm{a}_i\bm{a}_i^\top\right)\geq \alpha^{m} \sum_{T\in {[n]\choose m}}\prod_{i\in T} \hat{x}_i \det\left(\sum_{i\in T}\bm{a}_i\bm{a}_i^\top\right)=\alpha^{m}\det\left(\sum_{i\in [n]}\hat{x}_i\bm{a}_i\bm{a}_i^{\top}\right)
	\end{align*}
	where the first equality is because of definition of probability measure $\mu$, the second equality is due to Lemma~\ref{lem_key_lem2}, the third equality is due to the interchange of summation, the first inequality is due to Definition~\ref{def:m-alpha} and the fourth equality is because of Lemma~\ref{lem_key_lem2}.
	\QEDA\end{proof}


\subsection{Analysis of Sampling Scheme}\label{sec:rand_algo_1}

In this subsection, we will analyze the proposed sampling procedure, i.e., deriving the approximation ratio $\alpha$ of sampling procedure \eqref{eq:sampling1} in \Cref{lemma:reduction}. The key idea is to derive lower bound for the ratio
$$\frac{\Pr\left[T\subseteq\S\right]}{\prod_{i\in T} \hat{x}_i}$$
for any $T\in {[n]\choose m}$, where $(\hat{\bm x},\hat w)$ is an optimal solution to \eqref{opt_sensor_convex} with $B=[0,1]$.

By the definition of random set $\S$ in \eqref{eq:sampling1}, the probability $\Pr\left[T\subseteq\S\right]$ is equal to
\begin{equation}
\Pr[T\subseteq \S]=\frac{\sum_{S\in {[n]\choose k}:T\subseteq S} \prod_{j\in S}\hat{x}_j}{
	\sum_{\bar{S}\in {[n]\choose k}}\prod_{i\in \bar{S}}\hat{x}_i}.\label{eq_pr_t_s}
\end{equation}
Observe that the denominator in \eqref{eq_pr_t_s} is a degree $k$ polynomial that is invariant under any permutation of $[n]$. Moreover, the numerator is also invariant under any permutation of $T$ as well as any permutation of $[n]\setminus T$. These observations allow us to use inequalities on symmetric polynomials and reduce the worst-case ratio of $\Pr[T\subseteq \S]$ with $\prod_{i\in T} \hat{x}_i$ to a single variable optimization problem as shown in the following proposition. We then analyze the single variable optimization problem to prove the desired bound.
\begin{restatable}{proposition}{thmalgrandround}\label{thm_alg_rand_round} Let $\S$ be the random set defined in \eqref{eq:sampling1}. Then for any $T\subseteq [n]$ such that $|T|=m$, we have
	$$ \Pr[T\subseteq \S]\geq \frac{1}{g(m,n,k)} \prod_{i\in T} \hat{x}_i:=\alpha^m \prod_{i\in T} \hat{x}_i,$$ where
	\begin{align}
	g(m,n,k)&=\max_{y}\left\{\sum_{\tau=0}^{m}\frac{{n-m\choose k-\tau}}{(n-m)^{m-\tau}{n-m\choose k-m}}\frac{{m\choose\tau}}{m^{\tau}}\left(k-y\right)^{m-\tau}\left(y\right)^{\tau}:\frac{mk}{n}\leq y\leq m\right\}.\label{eq_def_G}
	\end{align}
\end{restatable}
\begin{proof}See~\Cref{proof_thm_alg_rand_round}.
\QEDA\end{proof}

Next, we derive the upper bound of $g(m,n,k)$ in \eqref{eq_def_G}, which is a single-variable optimization problem. To derive the desired results, we first observe that for any given $(m,k)$ with $m\leq k$, $g(m,n,k)$ is monotone non-decreasing in $n$. This motivates us to find an upper bound on $\lim_{n\rightarrow \infty}g(m,n,k)$, which leads to the proposition below.
\begin{restatable}{proposition}{thmalgrandroundbnd}\label{thm_alg_rand_round_bnd}For any $n\geq k\geq m$, we have
	\begin{align}
	\alpha^{-1}=\left[g(m,n,k)\right]^{\frac{1}{m}} \leq \lim_{\tau\rightarrow \infty}\left[g(m,\tau,k)\right]^{\frac{1}{m}} \leq \min\left\{e,1+\frac{k}{k-m+1}\right\}.
	\end{align}
\end{restatable}
\begin{proof} See Appendix~\ref{proof_thm_alg_rand_round_bnd}.
\QEDA\end{proof}

Finally, we present our first approximation result blow.
\begin{theorem}\label{theorem:main1}
	For any positive integers $m\leq k\leq n$, \Cref{alg_rand_round} yields $\frac{1}{e}$-approximation for the $D$-optimal design problem.
\end{theorem}
\begin{proof}
	The result directly follows from \Cref{lemma:reduction} and Proposition~\ref{thm_alg_rand_round_bnd} given that $n\geq k\geq m$.
\QEDA\end{proof}

We also note that when $k$ is large enough, \Cref{alg_rand_round} is a near $0.5$-approximation. 
\begin{corollary}\label{thm_alg_rand_round_bnd_cor1}
	Given $\epsilon\in (0,1)$ and positive integers $m\leq k\leq n$, if $k\geq \frac{m-1}{2\epsilon}$, then \Cref{alg_rand_round} yields $(0.5-\epsilon)$-approximation for the $D$-optimal design problem.
\end{corollary}
\begin{proof}For any $\epsilon\in (0,1)$, from Proposition~\ref{thm_alg_rand_round_bnd}, let the lower bound approximation ratio $\alpha\geq [\min\{e,1+\frac{k}{k-m+1}\}]^{-1} \geq 0.5-\epsilon$, or equivalently, let
	\[1+\frac{k}{k-m+1}\leq \frac{1}{0.5-\epsilon}.\]
	Then the conclusion follows.
\QEDA\end{proof}


\subsection{Deterministic Implementation}\label{sec_deter_implem}

The approximation ratios presented in the previous subsection only hold in the sense of expectation. In this subsection, we will overcome this issue and present a deterministic Algorithm~\ref{alg_derand_round} with the same approximation guarantees. The key idea is to derandomize Algorithm~\ref{alg_derand_round} using the method of conditional expectation (cf., \cite{spencer1987ten}), and the main challenge is how to compute the conditional expectation efficiently. We next will show that it can be done by evaluating a determinant of a $n\times n$ matrix whose entries are linear polynomials in three indeterminates.


In this deterministic Algorithm~\ref{alg_derand_round}, suppose we are given a subset $S, P\subseteq [n]$, where $S$ is a chosen subset and $P$ is a unchosen set
such that $|S|=s\leq k$. Then the expectation of $m$th power of objective function given $S$ and $P$ is
\begin{align}
&H(S, P):=\E\left[\det\left(\sum_{i\in \S}\bm{a}_i\bm{a}_i^{\top}\right)\bigg| S\subseteq \S,P\cap  \S=\emptyset \right]\notag\\
&=\sum_{U\in {[n]\setminus (S\cup P)\choose k-s}}\frac{\prod_{j\in U}\hat{x}_j}{\sum_{\bar{U}\in {[n]\setminus (S\cup P)\choose k-s}}\prod_{i\in \bar{U}}\hat{x}_i}\det\left(\sum_{i\in U}\bm{a}_i\bm{a}_i^\top+\sum_{i\in S}\bm{a}_i\bm{a}_i^\top\right)\notag\\
&=\left(\sum_{\bar{U}\in {[n]\setminus (S\cup P)\choose k-s}}\prod_{i\in \bar{U}}\hat{x}_i\right)^{-1}\sum_{U\in {[n]\setminus (S\cup P)\choose k-s}}\prod_{j\in U}\hat{x}_j\sum_{R\in {U\cup S\choose m}}\det\left(\sum_{i\in R}\bm{a}_i\bm{a}_i^\top\right)\notag\\
&=\frac{\sum_{R\in {[n]\setminus P\choose m}, r:=|R\setminus S|\leq k-s}\prod_{j\in R\setminus S}\hat{x}_j\det\left(\sum_{i\in R}\bm{a}_i\bm{a}_i^\top\right)\sum_{W\in {[n]\setminus (S\cup P\cup  R)\choose k-s-r}}\prod_{j\in W}\hat{x}_j}{\sum_{\bar{U}\in {[n]\setminus (S\cup P)\choose k-s}}\prod_{i\in \bar{U}}\hat{x}_i},
\label{eq_H(S,T)}
\end{align}
where the second equality is a direct computing of the conditional probability, the third equality is due to Lemma~\ref{lem_key_lem} and the last one is because of interchange of summation.

\begin{algorithm}[htbp]
	\caption{Derandomization of \Cref{alg_rand_round}}
	\label{alg_derand_round}
	\begin{algorithmic}[1]
		\State Suppose $(\hat{\bm x},\hat w)$ is an optimal solution to the convex relaxation \eqref{opt_sensor_convex} with $B=[0,1]$, where $\hat{\bm x} \in [0,1]^n$ with $\sum_{i\in [n]}\hat{x} _i=k$ and $\hat{w}=f(\hat{\bm x})$
		\State Initialize chosen set $S=\emptyset$, unchosen set $P= \emptyset$, and $j=1$
		\Do
\State Compute $H(S\cup{j}, P), H(S, P\cup{j})$ as defined in \eqref{eq_H(S,T)}, and their denominators and numerators can be computed by~\Cref{lem_key_lem3} and \Cref{thm_derandom}, respectively
		\If {$H(S\cup{j}, P)>H(S, P\cup{j})$}
		\State Add $j$ to set $S$
\Else
\State Add $j$ to set $P$
\EndIf
\State $j=j+1$
		\doWhile{$|S|< k$ and $j\leq n$}
		\State Output $S$
	\end{algorithmic}
\end{algorithm}

Note that the denominator in \eqref{eq_H(S,T)} can be computed efficiently according to Observation~\ref{lem_key_lem3}. Next, we show that the numerator in \eqref{eq_H(S,T)} can be also computed efficiently.
\begin{proposition}\label{thm_derandom} Let matrix $\bm{A}=[\bm{a}_1,\ldots,\bm{a}_n]$. Consider the following function
	\begin{align}
	F(t_1,t_2,t_3)=\det\left(\bm{I}_n+t_1\diag(\bm{y})^{\frac{1}{2}}\bm{A}^{\top}\bm{A}\diag(\bm{y})^{\frac{1}{2}}+\diag(\bm{y})\right),\label{eq_F}
	\end{align} where $t_1,t_2,t_3\in \Re,\bm{y}\in \Re^n$ are indeterminate and
	\begin{align*}
	y_i=\left\{\begin{array}{cc}t_3, &\text{ if }i\in S\\
0, &\text{ if }i\in P\\
	\hat{x}_it_2, &\text{otherwise}
	\end{array}\right..
	\end{align*}
	Then, the coefficient of $t_1^{m}t_2^{k-s}t_3^{s}$ in $F(t_1,t_2,t_3)$ equals to
	\begin{align}\sum_{R\in {[n]\setminus P\choose m}, r:=|R\setminus S|\leq k-s}\prod_{j\in R\setminus S}\hat{x}_j\det\left(\sum_{i\in R}\bm{a}_i\bm{a}_i^\top\right)\sum_{W\in {[n]\setminus (S\cup P\cup  R)\choose k-s-r}}\prod_{j\in W}\hat{x}_j.\label{eq_coeff_t_1_t_2_t_3}
	\end{align}
\end{proposition}
\begin{proof}
	First of all, we can rewrite $F(t_1,t_2,t_3)$ as
	\begin{align*}
	F(t_1,t_2,t_3)&=\det\left(\bm{I}_n+\diag(\bm{y})\right)\det\left(\bm{I}_n+t_1\diag(\e+\bm{y})^{-\frac{1}{2}}\diag(\bm{y})^{\frac{1}{2}}\bm{A}^{\top}\bm{A}\diag(\bm{y})^{\frac{1}{2}}\diag(\e+\bm{y})^{-\frac{1}{2}}\right)\\
	&=\prod_{i\in S}\left(1+t_3\right)\prod_{i\in [n]\setminus (S\cup P)}\left(1+\hat{x}_it_2\right)\det\left(\bm{I}_n+t_1\bm{W}^{\top}\bm{W}\right)
	\end{align*}
	where the $i$th column of matrix $\bm{W}$ is
	\begin{align*}
	\bm{W}_i=\left\{\begin{array}{cc}\sqrt{\frac{t_3}{1+t_3}}\bm{a}_i, &\text{ if }i\in S\\
\bm{0}, &\text{ if }i\in P\\
	\sqrt{\frac{\hat{x}_it_2}{1+\hat{x}_it_2}}\bm{a}_i, &\text{otherwise}
	\end{array}\right..
	\end{align*}
	Note that the coefficient of $t_1^{m}$ in $\det\left(\bm{I}_n+t_1\bm{W}_{[n]\setminus P}^{\top}\bm{W}_{[n]\setminus P}\right)$ is equal to the one of $\prod_{i\in [n]\setminus P}(1+t_1\bm{\Lambda}_i)$, where $\{\bm{\Lambda}_i\}_{i\in [n]\setminus P}$ are the eigenvalues of $\bm{W}^{\top}\bm{W}$. Thus, the coefficient of $t_1^{m}$ is
	\[\sum_{R\in {[n]\setminus P\choose m}}\prod_{i\in R}\bm{\Lambda}_i=\sum_{R\in {[n]\setminus P\choose m}}\det\left((\bm{W}^{\top}\bm{W})_{R,R}\right)=\sum_{R\in {[n]\setminus P\choose m}}\det\left(\sum_{i\in R}\bm{W}_i\bm{W}_i^{\top}\right)=\sum_{R\in {[n]\setminus P\choose m}}\det\left(\sum_{i\in R}\bm{W}_i\bm{W}_i^{\top}\right)\]
	where $\bm{P}_{R_1,R_2}$ denotes a submatrix of $\bm{P}$ with rows and columns from sets $R_1,R_2$, the first equality is due to the property of the eigenvalues (Theorem 1.2.12 in \cite{horn1985matrix}), and the second inequality is because the length of each column of $\bm{W}$ is $m$, and the third equality is because the determinant of singular matrix is $0$.
	
	Therefore, the coefficient of $t_1^{m}t_2^{k-s}t_3^{s}$ in $F(t_1,t_2,t_3)$ is equivalent to the one of
	\[\prod_{i\in S}\left(1+t_3\right)\prod_{i\in [n]\setminus (S\cup P)}\left(1+\hat{x}_it_2\right)\sum_{R\in {[n]\setminus P\choose m}}\det\left(\sum_{i\in R}\bm{W}_i\bm{W}_i^{\top}\right).\]
	By Lemma~\ref{lem_key_lem2} with $n=m$ and the definition of matrix $\bm{W}$, the coefficient of $t_1^{m}t_2^{k-s}t_3^{s}$ in $F(t_1,t_2,t_3)$ is further equivalent to the one of
	\begin{align*}
	&\prod_{i\in S}\left(1+t_3\right)\prod_{i\in [n]\setminus (S\cup P)}\left(1+\hat{x}_it_2\right)\sum_{R\in {[n]\setminus P\choose m}}t_1^m \prod_{j\in R\setminus S}\frac{\hat{x}_j}{1+t_2\hat{x}_j}\prod_{j\in R\cap S}\frac{t_3}{1+t_3}\det\left(\sum_{i\in R}\bm{a}_i\bm{a}_i^{\top}\right)\\
	&=t_1^m\sum_{R\in {[n]\setminus P\choose m}}t_2^{|R\setminus S|}t_3^{|R\cap S|}\left(1+t_3\right)^{|S\setminus R|}\prod_{i\in [n]\setminus (S\cup R)}\left(1+\hat{x}_it_2\right) \prod_{j\in R\setminus S}\hat{x}_j \det\left(\sum_{i\in R}\bm{a}_i\bm{a}_i^{\top}\right)
	\end{align*}
	which is equal to \eqref{eq_coeff_t_1_t_2_t_3} by collecting coefficients of $t_1^{m}t_2^{k-s}t_3^{s}$.
\QEDA\end{proof}
{Note that the characteristic function of a matrix can be computed efficiently by using Faddeev-LeVerrier algorithm \cite{hou1998classroom}. Thus, the polynomial $F(t_1,t_2,t_3)$ is efficiently computable with time complexity of $O(n^4)$.}

Algorithm \ref{alg_derand_round} proceeds as follows. We start with an empty subset $S$  of chosen elements and an empty subset $P$  of unchosen elements, and for each $j\in [n]\setminus (S\cup P)$, we compute two expected $m$th power of objective functions that $j$ will be chosen $H(S\cup{j}, P)$ or $j$ will be not chosen $H(S, P\cup{j})$. We add $j$ to $S$ if $H(S\cup{j}, P)>H(S, P\cup{j})$. Otherwise, we add $j$ to $P$. Then go to next iteration. This procedure will terminate if $|S|=k$ or $S\cup P=[n]$. {Note that Algorithm~\ref{alg_derand_round} requires $O(n)$ evaluations of function $H(S, P)$, thus its the time complexity is $O(n^5)$. Hence, in practice, we recommend Algorithm \ref{alg_rand_round} due to its shorter running time.}

The approximation results for Algorithm~\ref{alg_derand_round} are identical to Theorem~\ref{theorem:main1} and Corollary~\ref{thm_alg_rand_round_bnd_cor1}, which are summarized as follows:
\begin{theorem}\label{theorem:main1_det}
	For any positive integers $m\leq k\leq n$,
	\begin{enumerate}[(i)]
		\item deterministic \Cref{alg_derand_round} is efficiently computable and yield $\frac{1}{e}$-approximation for the $D$-optimal design problem; and
		\item given $\epsilon\in (0,1)$, if $k\geq \frac{m-1}{2\epsilon}$, then deterministic \Cref{alg_derand_round} yields $(0.5-\epsilon)$-approximation for the $D$-optimal design problem.
	\end{enumerate}
\end{theorem}

\section{Improving Approximation Bound in Asymptotic Regime}\label{proof-theorem:main2}


In this section, we propose another sampling Algorithm \ref{alg_rand_round_asymp} which achieves asymptotic optimality, i.e. the output of Algorithm \ref{alg_rand_round_asymp} is close to be optimal when $k/m \rightarrow \infty$. We also show the derandomization of \Cref{alg_rand_round_asymp}. Similar to previous section, in this section, we consider $D$-optimal design problem without repetitions, i.e., in \eqref{opt_sensor} and \eqref{opt_sensor_convex}, we let $B=[0,1]$.

In Algorithm~\ref{alg_rand_round_asymp}, suppose that $(\hat{\bm x},\hat w)$ is an optimal solution to the convex relaxation \eqref{opt_sensor_convex} with $B=[0,1]$, $\epsilon\in (0,1)$ is a positive threshold and $\N$ is a random permutation  of $[n]$. Then for each $j\in \N$, we select $j$ with probability $\frac{x_j}{1+\epsilon}$, and let $\S$ be the set of selected elements. If $|\S|<k$, then we can add $k-|\S|$ more elements from $[n]\setminus \S$. On the other hand, if $|\S|>k$, then we repeat the sampling procedure. {Algorithm~\ref{alg_rand_round_asymp} has time complexity  $O(n)$. In addition, note that the difference between Algorithm~\ref{alg_rand_round} and Algorithm~\ref{alg_rand_round_asymp} is that in Algorithm~\ref{alg_rand_round_asymp}, we inflate the probability of choosing $j$th experiment by $\frac{1}{1+\epsilon}$. This condition guarantees that when $k\gg m$, according to concentration inequality, the probability of size-$m$ subset $T$ to be chosen will be nearly equal to $\prod_{j\in T}\frac{x_j}{1+\epsilon}$.}

\begin{algorithm}[htbp]
	\caption{Asymptotic Sampling Algorithm}
	\label{alg_rand_round_asymp}
	\begin{algorithmic}[1]
		\State Suppose $(\hat{\bm x},\hat w)$ is an optimal solution to the convex relaxation \eqref{opt_sensor_convex} with $B=[0,1]$, where $\hat{\bm x} \in [0,1]^n$ with $\sum_{i\in [n]}\hat{x} _i=k$ and $\hat{w}=f(\hat{\bm x})$
		\State Initialize $\S=\emptyset$ and a positive number $\epsilon>0$
		\Do
		\State Let set $\N$ be a random permutation set of $\{1,\ldots,n\}$
		\For{$j\in \N$}
		\State Sample a $(0,1)$ uniform random variable $U$
		\If{$U\leq \frac{\hat{x}_i}{1+\epsilon}$ }
		\State  Add $j$ to set $\S$
		\EndIf
		\EndFor
		\doWhile{$|\S|>k$}
		\If{$|\S|<k$}\Comment{Greedy step to enforce $|\S|=k$}
		\State Let $j^*\in \arg\max_{j\in [n]\setminus \S} \left[\det\left(\sum_{i\in \S}\bm{a}_i\bm{a}_i^\top+\bm{a}_j\bm{a}_j^\top\right)\right]^{\frac{1}{m}}$
		\State Add $j^*$ to set $\S$
		\EndIf
		\State Output $\S$
	\end{algorithmic}
\end{algorithm}

\subsection{Analysis of Sampling Algorithm~\ref{alg_rand_round_asymp}}
To analyze sampling Algorithm~\ref{alg_rand_round_asymp}, we first show the following probability bound. The key idea is to prove the lower bound $\frac{1}{\prod_{i\in T}\hat{x}_i}\Pr\left\{T\subseteq \S\bigg||\S|\leq k\right\}$ by using Chernoff Bound \cite{chernoff1952measure}.
\begin{lemma}\label{lem_sensor_asymp}Let $\epsilon>0$ and $\S\subseteq [n]$ be a random set output from Algorithm~\ref{alg_rand_round_asymp}. Given $T\subseteq [n]$ with $|T|=m\leq n$, then we have
	\begin{align}
	&\alpha^m:=\frac{1}{\prod_{i\in T}\hat{x}_i}\Pr\left\{T\subseteq \S\bigg||\S|\leq k\right\}\geq (1+\epsilon)^{-m}\left(1-e^{-\frac{\left(\epsilon k -(1+\epsilon)m\right)^2}{k(2+\epsilon)(1+\epsilon)}}\right),\label{eq_pr_asym}
	\end{align}
	where $\alpha$ is in Definition~\ref{def:m-alpha}. In addition, when $ k\geq \frac{4m}{\epsilon}+\frac{12}{\epsilon^2}\log(\frac{1}{\epsilon})$, then
	\begin{align}
	&\alpha^m \geq\left(1-\epsilon\right)^m.\label{eq_pr_asym2}
	\end{align}
	
\end{lemma}
\begin{proof} We note that $\S\subseteq [n]$ is a random set, where each $i\in [n]$ is independently sampled according to Bernoulli random variable $X_i$ with the probability of success $\frac{\hat{x}_i}{1+\epsilon}$. According to Definition~\ref{def:m-alpha} and by ignoring the greedy procedure Algorithm~\ref{alg_rand_round_asymp}, it is sufficient to derive the lower bound of $\frac{1}{\prod_{i\in T}\hat{x}_i}\Pr\left\{T\subseteq \S\bigg||\S|\leq k\right\}$, i.e.,
	\begin{align*}
	&\frac{\Pr\left\{ T\subseteq \S \bigg||\S|\leq k\right\}}{\prod_{i\in T}\hat{x}_i}=
	\frac{\Pr\left\{ T\subseteq \S , |\S|\leq k \right\}}{\prod_{i\in T}\hat{x}_i\Pr\left\{|\S|\leq k \right\}}
	\geq (1+\epsilon)^{-m}\Pr\left\{ \sum_{i\in [n]\setminus T}X_i \leq k-m \right\}
	\end{align*}
	where the first inequality is due to $\Pr\{|\S|\leq k \}\leq 1$.
	
	Therefore, it is sufficient to bound the following probability
	\begin{align*}
	&\Pr\left\{ \sum_{i\in [n]\setminus T}X_i \leq k-m \right\}.
	\end{align*}
	Since $X_i\in \{0,1\}$ for each $i\in [n]$, and $\E[\sum_{i\in [n]\setminus T}X_i]=\frac{1}{1+\epsilon}\sum_{i\in [n]\setminus T}\hat{x}_i$. According to Chernoff bound \cite{chernoff1952measure}, we have
	\begin{align*}
	&\Pr\left\{\sum_{i\in [n]\setminus T}X_i > (1+\bar{\epsilon}) \E\left[\sum_{i\in [n]\setminus T}X_i\right]\right\} \leq e^{-\frac{\bar{\epsilon}^2}{2+\bar{\epsilon}}\E\left[\sum_{i\in [n]\setminus T}X_i\right]}.
	\end{align*}
	where $\bar\epsilon$ is a positive constant. Therefore, by choosing $\bar{\epsilon}=\frac{(1+\epsilon)(k-m)}{\sum_{i\in [n]\setminus T}\hat{x}_i}-1$, we have
	\begin{align}
	&(1+\epsilon)^{-m}\Pr\left\{ \sum_{i\in [n]\setminus T}X_i \leq k-m \right\}\geq (1+\epsilon)^{-m}\left(1-e^{-\frac{\bar{\epsilon}^2\sum_{i\in [n]\setminus T}\hat{x}_i}{(2+\bar{\epsilon})(1+\epsilon)}}\right).\label{eq_pr_asym3}
	\end{align}
	
	Note that $k-m\leq \sum_{i\in [n]\setminus T}X_i \leq k $, $\epsilon -(1+\epsilon)\frac{m}{k}\leq \bar{\epsilon} \leq \epsilon$ and $\epsilon k -(1+\epsilon)m\leq \bar{\epsilon}\sum_{i\in [n]\setminus T}\hat{x}_i \leq \epsilon(k-m)$. Suppose that $k\geq \frac{1+\epsilon}{\epsilon}m$, then the left-hand side of \eqref{eq_pr_asym3} can be further lower bounded as
	\begin{align*}
	&(1+\epsilon)^{-m}\left(1-e^{-\frac{\bar{\epsilon}^2\sum_{i\in [n]\setminus T}\hat{x}_i}{(2+\bar{\epsilon})(1+\epsilon)}}\right)
	\geq (1+\epsilon)^{-m}\left(1-e^{-\frac{\left(\epsilon k -(1+\epsilon)m\right)^2}{k(2+\epsilon)(1+\epsilon)}}\right).
	\end{align*}
	
	To prove\eqref{eq_pr_asym2}, it remains to show
	\begin{align*}
	&1-e^{-\frac{\left(\epsilon k -(1+\epsilon)m\right)^2}{k(2+\epsilon)(1+\epsilon)}}\geq \left((1-\epsilon)(1+\epsilon)\right)^m,
	\end{align*}
	or equivalently,
	\begin{align}
	&\log\left[1-\left(1-\epsilon^2\right)^m\right]\geq -\frac{\left(\epsilon k -(1+\epsilon)m\right)^2}{k(2+\epsilon)(1+\epsilon)},\label{eq_bound_k_m}
	\end{align}
	which holds if
	\[k\geq \left(1+\frac{1}{\epsilon}\right)\left(2m-\left(1+\frac{2}{\epsilon}\right)\log\left[1-\left(1-\epsilon^2\right)^m\right]\right).\]
	
	We note that $-\log\left[1-\left(1-\epsilon^2\right)^m\right]$ is non-increasing over $m\geq 1$, therefore is upper bounded by $2\log(\frac{1}{\epsilon})$. Hence, \eqref{eq_bound_k_m} holds if $k\geq \frac{4m}{\epsilon}+\frac{12}{\epsilon^2}\log(\frac{1}{\epsilon})$.
	
\QEDA\end{proof}


Finally, we state our main approximation result as below.
\begin{theorem}\label{theorem:main2}
For any integers $m\leq k\leq n$ and $\epsilon \in (0,1)$, if $k\geq \frac{4m}{\epsilon}+\frac{12}{\epsilon^2}\log(\frac{1}{\epsilon})$, then \Cref{alg_rand_round_asymp} is a $(1-\epsilon)$-approximation for the $D$-optimal design problem.
\end{theorem}
\begin{proof}The result directly follows from Lemmas~\ref{lemma:reduction} and~\ref{lem_sensor_asymp}.
\QEDA\end{proof}

\subsection{Deterministic Implementation}\label{sec_deter_implem_a}

Similar to \Cref{sec_deter_implem}, the approximation ratios presented in the previous subsection only hold in the sense of expectation. In this subsection, we will overcome this issue and present a deterministic Algorithm~\ref{alg_rand_round_asymp} with the same approximation guarantees.


In this deterministic Algorithm~\ref{alg_derand_round_a}, $S$ denotes a subset such that $|S|=s\leq k$ and $P$ denotes the unchosen set. Then the expectation of $m$th power of objective function given $S$ and $P$ is
\begin{align}
&H(S,P):=\E\left[\det\left(\sum_{i\in \S}\bm{a}_i\bm{a}_i^{\top}\right)\bigg| S\subseteq \S, |\S|\leq k, \S\cap P=\emptyset\right]\notag\\
&=\frac{\sum_{\kappa=s}^{k}\sum_{U\in {[n]\setminus (S\cup P)\choose \kappa-s}}\prod_{j\in U}\frac{\hat{x}_i}{1+\epsilon}\prod_{j\in [n]\setminus (S\cup P\cup U)}\left(1-\frac{\hat{x}_i}{1+\epsilon}\right)\det\left(\sum_{i\in U}\bm{a}_i\bm{a}_i^\top+\sum_{i\in S}\bm{a}_i\bm{a}_i^\top\right)}{\sum_{\kappa=1}^{k}\sum_{\bar{U}\in {[n]\setminus (S\cup P)\choose \kappa-s}}\prod_{i\in \bar{U}}\frac{\hat{x}_i}{1+\epsilon}\prod_{j\in [n]\setminus (S\cup P\cup \bar{U})}\left(1-\frac{\hat{x}_i}{1+\epsilon}\right)}\notag\\
&=\left(\sum_{\kappa=s}^{k}\sum_{\bar{U}\in {[n]\setminus (S\cup P)\choose \kappa-s}}\prod_{i\in \bar{U}}\frac{\hat{x}_i}{1+\epsilon-\hat{x}_i}\right)^{-1}\sum_{\kappa=s}^{k}\sum_{U\in {[n]\setminus (S\cup P)\choose \kappa-s}}\prod_{j\in U}\frac{\hat{x}_i}{1+\epsilon-\hat{x}_i}\sum_{R\in {U\cup S\choose m}}\det\left(\sum_{i\in R}\bm{a}_i\bm{a}_i^\top\right)\notag\\
&=\frac{\sum_{\kappa=s}^{k}\sum_{R\in {[n]\setminus P\choose m}, r:=|R\setminus S|\leq \kappa-s}\prod_{j\in R\setminus S}\frac{\hat{x}_j}{1+\epsilon-\hat{x}_j}\det\left(\sum_{i\in R}\bm{a}_i\bm{a}_i^\top\right)\sum_{W\in {[n]\setminus (S\cup P\cup R)\choose \kappa-s-r}}\prod_{j\in W}\frac{\hat{x}_j}{1+\epsilon-\hat{x}_j}}{\sum_{\kappa=s}^{k}\sum_{\bar{U}\in {[n]\setminus (S\cup P)\choose \kappa-s}}\prod_{i\in \bar{U}}\frac{\hat{x}_i}{1+\epsilon-\hat{x}_i}}
,\label{eq_H(S,T)_a}
\end{align}
where the second equality is a direct computing of the conditional probability, the third equality is due to Lemma~\ref{lem_key_lem}, dividing both denominator and numerator by $\prod_{j\in [n]\setminus (S\cup P)}\left(1-\frac{\hat{x}_i}{1+\epsilon}\right)$ and the convention, set ${S\choose \tau}=\emptyset$ if $\tau>|S|$ or $\tau<0$ and the fourth one is because of interchange of summation.

Note that the $\kappa$th entry with $\kappa \in \{s,\ldots,k\}$ of the denominator in \eqref{eq_H(S,T)_a} can be computed efficiently according to Observation~\ref{lem_key_lem3} by letting $x_i:=\frac{\hat{x}_i}{1+\epsilon-\hat{x}_i}$. Meanwhile, $\kappa$th entry with $\kappa \in \{s,\ldots,k\}$ of the numerator in \eqref{eq_H(S,T)} can be also computed efficiently by \Cref{thm_derandom} by letting $\hat{x}_i:=\frac{\hat{x}_i}{1+\epsilon-\hat{x}_i}$. Therefore, the conditional expectation in \eqref{eq_H(S,T)} is efficiently computable

In summary, Algorithm \ref{alg_derand_round_a} proceeds as follows. 
We start with an empty subset $S$  of chosen elements and an empty subset $P$  of unchosen elements, and for each $j\in [n]\setminus (S\cup P)$, we compute two expected $m$th power of objective functions that $j$ will be chosen $H(S\cup{j}, P)$ or $j$ will be not chosen $H(S, P\cup{j})$. We add $j$ to $S$ if $H(S\cup{j}, P)>H(S, P\cup{j})$. Otherwise, we add $j$ to $P$. Then go to next iteration. This procedure will terminate if $|S|=k$ or $S\cup P=[n]$. 
{Similar to Algorithm~\ref{alg_derand_round}, the time complexity of Algorithm~\ref{alg_derand_round_a} is $O(n^5k)$. Hence, in practice, we recommend the more efficient Algorithm \ref{alg_rand_round_asymp}.}

The approximation result for Algorithm~\ref{alg_derand_round_a}  is identical to Theorem~\ref{theorem:main2}, which is summarized as follows:
\begin{theorem}\label{theorem:main1_det_a}
	For any positive integers $m\leq k\leq n$ and $\epsilon \in (0,1)$, deterministic \Cref{alg_derand_round_a} is efficiently computable and yield $1-\epsilon$-approximation for the $D$-optimal design problem if $k\geq \frac{4m}{\epsilon}+\frac{12}{\epsilon^2}\log(\frac{1}{\epsilon})$.
\end{theorem}

\begin{algorithm}[htbp]
	\caption{Derandomization of \Cref{alg_rand_round_asymp}}
	\label{alg_derand_round_a}
	\begin{algorithmic}[1]
		\State Suppose $(\hat{\bm x},\hat w)$ is an optimal solution to the convex relaxation \eqref{opt_sensor_convex} with $B=[0,1]$, where $\hat{\bm x} \in [0,1]^n$ with $\sum_{i\in [n]}\hat{x} _i=k$ and $\hat{w}=f(\hat{\bm x})$
		\State Initialize chosen set $S=\emptyset$, unchosen set $P= \emptyset$, and $j=1$
		\Do
\State Compute $H(S\cup{j}, P), H(S, P\cup{j})$ as defined in \eqref{eq_H(S,T)_a}, and their denominators and numerators can be computed by~\Cref{lem_key_lem3} and \Cref{thm_derandom}, respectively
		\If {$H(S\cup{j}, P)>H(S, P\cup{j})$}
		\State Add $j$ to set $S$
\Else
\State Add $j$ to set $P$
\EndIf
\State $j=j+1$
		\doWhile{$|S|< k$ and $j\leq n$}
		\State Output $S$
	\end{algorithmic}
\end{algorithm}

\section{Approximation Algorithm for $D$-optimal Design Problem with Repetitions}\label{proof-theorem:main3}
In this section, we consider the $D$-optimal design problem with repetitions, i.e., we let $B= \Re_+$ in \eqref{opt_sensor} and \eqref{opt_sensor_convex}. We will propose a new analysis of the algorithm proposed by \cite{nikolov2015randomized}, derive its approximation ratio and show its deterministic implementation. Again, in this section, we also let $(\hat{\bm x},\hat w)$ be an optimal solution to the convex relaxation \eqref{opt_sensor_convex}, where $\hat{\bm x} \in \Re_+^n$ with $\sum_{i\in [n]}\hat{x} _i=k$ and $\hat{w}=f(\hat{\bm x})$. Since the set of all nonnegative rational vectors is dense in the set of all nonnegative real vectors, thus without loss of generality, we assume that $\hat{\bm x}$ is a nonnegative rational vector (i.e., $\hat{\bm x}\in \Qe_+^n$).

In \cite{nikolov2015randomized}, the author suggested to obtain $k$-sample set $\S$ with replacement, i.e. $\S$ can be a multi-set. The sampling procedure can be separated into $k$ steps. At each step, a sample $s$ is selected with probability $\Pr\{s=i\}=\frac{\hat{x}_i}{k}$ (note that $\bm{\hat{x}} \in \Re_+^n$ with $\sum_{i\in [n]}\hat{x}_i=k$). The detailed description is in Algorithm~\ref{alg_rand_round_doe}. This sampling procedure can be interpreted as follows: let $\{X_i\}_{i\in [n]}$ be independent Poisson random variables where $X_i$ has arrival rate $\hat{x}_i$. We note that conditioning on total number of arrivals equal to $k$ (i.e., $\sum_{i\in [n]}X _i=k$), the distribution of $\{X_i\}_{i\in [n]}$ is multinomial (cf., \cite{albert2011dirichlet}), where there are $k$ trials and the probability of $i$th entry to be chosen is $\frac{\hat{x}_i}{k}$. We terminate this sampling procedure if the total number of arrivals equals to $k$. {Note that the time complexity of Algorithm~\ref{alg_rand_round_doe} is $O(n)$.}

\begin{algorithm}
	\caption{Sampling Algorithm for $D$-optimal Design with Repetitions}
	\label{alg_rand_round_doe}
	\begin{algorithmic}[1]
		\State Suppose $(\hat{\bm x},\hat w)$ is an optimal solution to the convex relaxation \eqref{opt_sensor_convex} with $B=\Re_+$, where $\hat{\bm x} \in \Qe_+^n$ with $\sum_{i\in [n]}\hat{x} _i=k$ and $\hat{w}=f(\hat{\bm x})$
		\State Initialize chosen multi-set $\S=\emptyset$
		\For{$j=1,\ldots,k$}
		\State Sample $s$ from $[n]$ with probability $\Pr\{s=i\}=\frac{\hat{x}_i}{k}$
		\State Let $\S=\S\cup\{s\}$
		\EndFor
		\State Output $\S$
	\end{algorithmic}
\end{algorithm}

To analyze Algorithm~\ref{alg_rand_round_doe}, let us consider another Algorithm~\ref{alg_rand_round_doe_ap}, which turns out to be arbitrarily close to \Cref{alg_rand_round_doe}. As $\hat{\bm x}$ is a nonnegative rational vector (i.e., $\hat{\bm x}\in \Qe_+^n$), we let $q$ be a common multiple of the denominators of rational numbers $\hat{x}_1,\ldots,\hat{x}_n$, i.e. $q\hat{x}_1,\ldots,q\hat{x}_n \in \Ze_+$. Next, we create a multi-set $\A_q$, which contains $q\hat{x}_i$ copies of index $i$ for each $i\in [n]$, i.e. $|\A_q|=qk$. Finally, we sample a subset $\S_q$ of $k$ items from set $\A_q$ uniformly, i.e. with probability ${qk\choose k}^{-1}$. The detailed description is in Algorithm~\ref{alg_rand_round_doe_ap}. In this case, the sampling procedure has the following interpretation. Since sum of i.i.d. Bernoulli random variables is Binomial, hence we let $\{X_i'\}_{i\in [n]}$ be independent binomial random variables where $X_i'$ has number of trials $q\hat{x}_i$ and probability of success $\frac{1}{q}$ for each $i\in [n]$. We terminate the sampling procedure if the total number of succeeded trials equals to $k$.

The following lemma shows that the probability distributions of outputs of Algorithms~\ref{alg_rand_round_doe} and \ref{alg_rand_round_doe_ap} can be arbitrarily close.
\begin{lemma}\label{lem_doe_ap} Let $\S$ and $\S_q$ be outputs of Algorithms~\ref{alg_rand_round_doe} and \ref{alg_rand_round_doe_ap}, respectively. Then $\S_{q}\xrightarrow{\mu}\S$, i.e. the probability distribution of $\S_{q}$ converges to $\S$ as $q\rightarrow \infty$.
\end{lemma}
\begin{proof} Consider two classes of independent random variables $\{X_i\}_{i\in [n]}$, $\{X'_{i}\}_{i\in [n]}$, where $X_i$ is Poisson random variable with arrival rate $\hat{x}_i$ for each $i\in [n]$ and $X'_{i}$ is binomial random variable with number of trials $q\hat{x}_i$ and probability of success $\frac{1}{q}$ for each $i\in [n]$, respectively.
	
	Given a size-$k$ multi-set $\R$ with support $\Supp{\R}\subseteq [n]$ and $M_{\R}(i)$ denoting the number of occurrences of $i$ in $\R$, according to the description of \Cref{alg_rand_round_doe_ap}, we have
	\begin{align*}
	&\Pr\left\{\S=\R\right\} =\Pr\left\{X_i=M_{\R}(i),\forall i\in [n] \bigg\vert \sum_{i\in [n]}X_i=k\right\}
	=\frac{\Pr\left\{X_i=M_{\R}(i),\forall i\in [n], \sum_{i\in [n]}X_i=k\right\}}{\Pr\left\{\sum_{i\in [n]}X_i=k\right\}}\\
	&=\I\left(|\R|=k\right)\frac{\prod_{i\in [n]}\Pr\left\{X_i=M_{\R}(i)\right\}}{\Pr\left\{\sum_{i\in [n]}X_i=k\right\}},
	\end{align*}
	where the first equality is from the description of Algorithm~\ref{alg_rand_round_doe}, the second equality is by the definition of conditional probability, the third equality is because $\{X_i\}_{i\in [n]}$ are independent from each other and $\I(\cdot)$ denotes indicator function.
	Similarly, we also have
	\begin{align*}
	\Pr\left\{\S_q=\R\right\} =\I\left(|\R|=k\right)\frac{\prod_{i\in [n]}\Pr\left\{X_i'=M_{\R}(i)\right\}}{\Pr\left\{\sum_{i\in [n]}X_i'=k\right\}}.
	\end{align*}
	
	Followed by the well-known Poisson limit theorem (cf. \cite{papoulis2002probability}), $X_i$ and $X'_{i}$ have the same distribution as $q\rightarrow \infty$ for any $i\in [n]$. Therefore,
	\begin{align*}
	&\Pr\left\{\S_q=\R\right\}\rightarrow\Pr\left\{\S=\R\right\},
	\end{align*}
	when $q\rightarrow \infty$, i.e., the outputs of Algorithm~\ref{alg_rand_round_doe} and \ref{alg_rand_round_doe_ap} have the same distribution when $q\rightarrow \infty$.
	
	
\QEDA\end{proof}

\begin{algorithm}[tb]
	\caption{Approximation of Algorithm~\ref{alg_rand_round_doe}}
	\label{alg_rand_round_doe_ap}
	\begin{algorithmic}[1]
		\State Suppose $(\hat{\bm x},\hat w)$ is an optimal solution to the convex relaxation \eqref{opt_sensor_convex} with $B=\Re^+$, where $\hat{\bm x} \in \Qe_+^n$ with $\sum_{i\in [n]}\hat{x} _i=k$ and $\hat{w}=f(\hat{\bm x})$
		\State Let $q$ be a common multiple of the denominators of rational numbers $\hat{x}_1,\ldots,\hat{x}_n$, i.e. $q\hat{x}_1,\ldots,q\hat{x}_n \in \Ze_+$
		\State Duplicate $q\hat{x}_i$ copies of index $i$ for each $i\in [n]$ as set $\A_q$, i.e. $|\A_q|=qk$
		\State Sample a subset $\S_q$ of $k$ items from set $\A_q$ with probability ${qk\choose k}^{-1}$
		\State Output $\S_q$
	\end{algorithmic}
\end{algorithm}

Now we are ready to present our approximation results for Algorithm \ref{alg_rand_round_doe}. The proof idea is based on Lemma~\ref{lem_doe_ap}, i.e., we first analyze Algorithm \ref{alg_rand_round_doe_ap} and apply its result to Algorithm \ref{alg_rand_round_doe} by letting $q\rightarrow\infty$.
\begin{proposition}\label{thm_alg_rand_round_doe} Let $\S$ and $\S_q$ be outputs of Algorithms~\ref{alg_rand_round_doe} and \ref{alg_rand_round_doe_ap}, respectively. Then $$\left(\E[(f(\S_q))^m]\right)^{\frac{1}{m}} \geq \left(\E[(f(\S))^m]\right)^{\frac{1}{m}} \geq g(m,k)^{-1}w^*,$$
	where
	\begin{align}
	g(m,k)&=\left[\frac{(k-m)!k^m}{k!}\right]^{\frac{1}{m}} \leq\min\left\{e,\frac{k}{k-m+1}\right\}.
	\end{align}
\end{proposition}
\begin{subequations}
	\begin{proof}We will first show the approximation ratio of Algorithm \ref{alg_rand_round_doe_ap} and then apply it to Algorithm \ref{alg_rand_round_doe} by Lemma~\ref{lem_doe_ap} when $q\rightarrow\infty$.
		\begin{enumerate}[(i)]
			\item Let $(\bar{x}'_{q},\bar{w}'_q)$ be output of Algorithm \ref{alg_rand_round_doe_ap}, and . Similar to the proof of Theorem~\ref{thm_alg_rand_round}, we have
			\begin{align}
			\E[(\bar{w}'_q)^m]&=\sum_{S\in {\A_q\choose k}}\frac{1}{{qk\choose k}}\det\left(\sum_{i\in S}\bm{a}_i\bm{a}_i^\top\right)
			=\frac{q^m}{{qk\choose k}}\sum_{S\in {\A_q\choose k}}\frac{1}{q^m}\det\left(\sum_{i\in S}\bm{a}_i\bm{a}_i^\top\right)\notag\\\notag\\
			&=\frac{q^m}{{qk\choose k}}\sum_{S\in {\A_q\choose k}}\frac{1}{q^m}\sum_{T\in {S\choose m}}\det\left(\sum_{i\in T}\bm{a}_i\bm{a}_i^\top\right)=\frac{q^m{qk-m\choose k-m}}{{qk\choose k}}\sum_{T\in {\A_q\choose m}}\frac{1}{q^m}\det\left(\sum_{i\in T}\bm{a}_i\bm{a}_i^\top\right)\notag\\
			&=\frac{q^m{qk-m\choose k-m}}{{qk\choose k}}\det\left(\sum_{i\in [n]}\hat{x}_i\bm{a}_i\bm{a}_i^{\top}\right) \geq \frac{k!}{(k-m)!k^m} \left[f(\bm{\hat{x}})\right]^m \geq \frac{k!}{(k-m)!k^m} \left(w^*\right)^m \notag
			\end{align}
			where the first and second equalities are due to Algorithm \ref{alg_rand_round_doe_ap}, the third equality is because of Lemma~\ref{lem_key_lem} and $k\geq m$, the fourth equality is due to interchange of summation, the fifth equality is because of the identity $\sum_{i\in [n]}\hat{x}_i\bm{a}_i\bm{a}_i^{\top}=\sum_{i\in \A_q}\frac{1}{q}\bm{a}_i\bm{a}_i^{\top}$, the first inequality holds because $\frac{(qk)^m(qk-m)!}{(qk)!}\geq 1$, and the last inequality is because $\bm{\hat{x}}$ is an optimal solution of the continuous relaxation.
			\item From Lemma~\ref{lem_doe_ap}, we know that the output of Algorithm \ref{alg_rand_round_doe_ap} has the same probability distribution as the output of Algorithm \ref{alg_rand_round_doe} when $q\rightarrow\infty$. Thus, we have
			\begin{align}
			\E[(\bar{w})^m]&=\lim_{q\rightarrow \infty}\E[(\bar{w}'_q)^m]=\lim_{q\rightarrow \infty}\frac{q^m{qk-m\choose k-m}}{{qk\choose k}}\det\left(\sum_{i\in [n]}\hat{x}_i\bm{a}_i\bm{a}_i^{\top}\right) \notag\\
			&=\frac{k!}{(k-m)!k^m}\det\left(\sum_{i\in [n]}\hat{x}_i\bm{a}_i\bm{a}_i^{\top}\right)= \frac{k!}{(k-m)!k^m} \left[f(\bm{\hat{x}})\right]^m .\notag
			\end{align}
			
			\item Next, let
			\begin{align*}
			{g}(m,k)&=\left[\frac{(k-m)!k^m}{k!}\right]^{\frac{1}{m}},
			\end{align*}
			and we would like to investigate its bound.

			First note that
			\begin{align*}
			\log \left(\frac{{g}(m,k+1)}{{g}(m,k)}\right)=m\log\left(1+\frac{1}{k}\right)+\log\left(1-\frac{m}{k+1}\right)
			\end{align*}
			which is nondecreasing over $k\in [m,\infty)$. Thus,
			\begin{align*}
			\log \left(\frac{{g}(m,k+1)}{{g}(m,k)}\right)\leq \lim_{k'\rightarrow \infty}\log \left(\frac{{g}(m,k'+1)}{{g}(m,k')}\right) =0,
			\end{align*}
			i.e., ${g}(m,k) \leq {g}(m,m) =\left[\frac{m^m}{m!}\right]^{\frac{1}{m}}\leq e$.
			
			On the other hand, since $\frac{(k-m)!}{k!} \leq \frac{1}{(k-m+1)^m}$, thus
			\begin{align*}
			{g}(m,k)&\leq\left[\frac{k^m}{(k-m+1)^m}\right]^{\frac{1}{m}}=\frac{k}{k-m+1}.
			\end{align*}
			
			Hence, ${g}(m,k)\leq \min\left\{e,\frac{k}{k-m+1}\right\}$.
			
		\end{enumerate}
	\QEDA\end{proof}
\end{subequations}
From Proposition~\ref{thm_alg_rand_round_doe}, we note that when $k$ is large enough, the output of Algorithm~\ref{alg_rand_round_doe} is almost optimal. 

Finally, we present our approximation results blow.
\begin{theorem}\label{theorem:main3}
	For any positive integers $m\leq k\leq n$,
	\begin{enumerate}[(i)]
		\item both \Cref{alg_rand_round_doe} and \Cref{alg_rand_round_doe_ap} yield $\frac{1}{e}$-approximation for the $D$-optimal design problem with repetitions; and
		\item given any $\epsilon\in (0,1)$, if $k\geq \frac{m-1}{\epsilon}$, then both \Cref{alg_rand_round_doe} and \Cref{alg_rand_round_doe_ap} yield $(1-\epsilon)$-approximation for the $D$-optimal design problem with repetitions.
	\end{enumerate}
\end{theorem}
\begin{proof}
	The first result directly follows from Proposition~\ref{thm_alg_rand_round_doe}. For the second one, given $\epsilon \in (0,1)$, by Proposition~\ref{thm_alg_rand_round_doe}, let
	\[\frac{k-m+1}{k}\geq 1-\epsilon.\]
	Then the conclusion follows by letting $k\geq \frac{m-1}{\epsilon}$.
\QEDA\end{proof}

To conclude this part, we remark that although the results from previous sections hold for $D$-optimal design with repetitions as well, \Cref{alg_rand_round_doe} has tighter approximation ratios. Therefore, investigating the convex relaxation solution and approximation algorithms of $D$-optimal design with repetitions alone does help us improve the approximation bounds.

\subsection{Deterministic Implementation}\label{sec_deter_implem3}
Similar to \Cref{sec_deter_implem}, the approximation ratios presented in the previous subsection hold in the sense of expectation. Recall that $(\hat{\bm x},\hat{w})$ is an optimal solution of \eqref{opt_sensor_convex} with $\hat{\bm x}\in \Qe_+^n$, $q$ is a common multiple of the denominators $\hat{x}_1,\ldots,\hat{x}_n$ and multi-set $\A_q$ of size $qk$ contains $q\hat{x}_i$ copies of index $i$ for each $i\in [n]$. In this subsection, we will show that the deterministic Algorithm~\ref{alg_derand_round_doe} applies to \Cref{alg_rand_round_doe}, which achieves the same approximation ratios.

In this deterministic Algorithm~\ref{alg_derand_round_doe}, let $S$ be a subset such that $|S|=s\leq k$. Let $\S_q,\S$ be outputs of \Cref{alg_rand_round_doe}, \Cref{alg_rand_round_doe_ap}, respectively. We know that from \Cref{lem_doe_ap}, $\S_{q}\xrightarrow{\mu}\S$ when $q\rightarrow\infty$. Thus, the expectation of $m$th power of objective function given $S$ is
\begin{align}
&H(S):=\E\left[\det\left(\sum_{i\in \S}\bm{a}_i\bm{a}_i^{\top}\right)\bigg| S\subseteq \S\right]=\lim_{q\rightarrow \infty}\E\left[\det\left(\sum_{i\in \S_q}\bm{a}_i\bm{a}_i^{\top}\right)\bigg| S\subseteq \S_q\right]\notag\\
&=\lim_{q\rightarrow \infty}\sum_{U\in {\A_q\setminus S\choose k-s}}\frac{\frac{1}{q^k}}{\sum_{\bar{U}\in {\A_q\setminus S\choose k-s}}\frac{1}{q^k}}\det\left(\sum_{i\in U}\bm{a}_i\bm{a}_i^\top+\sum_{i\in S}\bm{a}_i\bm{a}_i^\top\right)\notag\\
&=\lim_{q\rightarrow \infty}{qk-s\choose k-s}^{-1}\sum_{U\in {\A_q\setminus S\choose k-s}}\sum_{R\in {U\cup S\choose m}}\det\left(\sum_{i\in R}\bm{a}_i\bm{a}_i^\top\right)\notag\\
&=\lim_{q\rightarrow \infty}\sum_{r=1}^{\min\{k-s,m\}}\frac{q^r{qk-s-r\choose k-s-r}}{{qk-s\choose k-s}}\sum_{R\in {\A_q\choose m}, |R\setminus S|=r}q^{-r}\det\left(\sum_{i\in R}\bm{a}_i\bm{a}_i^\top\right)\notag\\
&=\sum_{r=1}^{\min\{k-s,m\}}\lim_{q\rightarrow \infty}\frac{q^r{qk-s-r\choose k-s-r}}{{qk-s\choose k-s}}\lim_{q\rightarrow \infty}\sum_{R\in {\A_q\choose m}, |R\setminus S|=r}q^{-r}\det\left(\sum_{i\in R}\bm{a}_i\bm{a}_i^\top\right),\notag\\
&=\sum_{r=1}^{\min\{k-s,m\}}\frac{(k-s)!}{k^r(k-s-r)!}\lim_{q\rightarrow \infty}\sum_{R\in {\A_q\choose m}, |R\setminus S|=r}q^{-r}\det\left(\sum_{i\in R}\bm{a}_i\bm{a}_i^\top\right)\label{eq_H(S,T)_doe}
\end{align}
where the second equality is due to $\S_{q}\xrightarrow{\mu}\S$ when $q\rightarrow\infty$, third equality is a direct computing of the conditional probability, the fourth equality is due to Lemma~\ref{lem_key_lem}, the fifth and last equalities is because both $\lim_{q\rightarrow \infty}\frac{q^r{qk-s-r\choose k-s-r}}{{qk-s\choose k-s}}$ and $\lim_{q\rightarrow \infty}\sum_{R\in {\A_q\choose m}, |R\setminus S|=r}q^{-r}$ $\det(\sum_{i\in R}\bm{a}_i\bm{a}_i^\top)$ exist and are finite for each $r \in \{1,\ldots,\min\{k-s,m\}\}$.

\begin{algorithm}[htbp]
	\caption{Derandomization of \Cref{alg_rand_round_doe_ap}}
	\label{alg_derand_round_doe}
	\begin{algorithmic}[1]
		\State Suppose $(\hat{\bm x},\hat w)$ is an optimal solution to the convex relaxation \eqref{opt_sensor_convex} with $B=[0,1]$, where $\hat{\bm x} \in [0,1]^n$ with $\sum_{i\in [n]}\hat{x} _i=k$ and $\hat{w}=f(\hat{\bm x})$
\State Initialize chosen set $S=\emptyset$
\Do
\State Let $j^*\in \arg\max_{j\in [n]\setminus \S} H(S\cup{j})$, where $H(S\cup{j})$ is define \eqref{eq_H(S,T)_doe}, and the limit can be computed by~\Cref{obser2}
\State Add $j^*$ to set $S$
\doWhile{$|S|< k$}
\State Output $S$
	\end{algorithmic}
\end{algorithm}

We note that from \Cref{lem_key_lem2}, for each $r=1,\ldots, \min\{k-s,m\}$, the limit in \eqref{eq_H(S,T)_doe} can be computed efficiently according to the following lemma.
\begin{lemma}\label{obser2}
For each $r=1,\ldots, \min\{k-s,m\}$,
\begin{enumerate}[(i)]
	\item the term $\sum_{R\in {\A_q\choose m}, |R\setminus S|=r}q^{-r}\det\left(\sum_{i\in R}\bm{a}_i\bm{a}_i^\top\right)$ is equal to the coefficient $t^r$ of the following determinant function
	\[\det\left(\frac{t}{q}\sum_{i\in \A_q\setminus S}\bm{a}_i\bm{a}_i^\top+\sum_{i\in S}\bm{a}_i\bm{a}_i^\top\right);\]
	\item the term $\lim_{q\rightarrow \infty}\sum_{R\in {\A_q\choose m}, |R\setminus S|=r}q^{-r}\det\left(\sum_{i\in R}\bm{a}_i\bm{a}_i^\top\right)$ is equal to the coefficient $t^r$ of the following determinant function
	\[\det\left(t\sum_{i\in [n]}\hat{x}_i\bm{a}_i\bm{a}_i^\top+\sum_{i\in S}\bm{a}_i\bm{a}_i^\top\right).\]
\end{enumerate}
\end{lemma}
\begin{proof}
	\begin{enumerate}[(i)]
		\item The results follow directly by \Cref{lem_key_lem2}.
		\item By part (i), the term $\lim_{q\rightarrow \infty}\sum_{R\in {\A_q\choose m}, |R\setminus S|=r}q^{-r}\det\left(\sum_{i\in R}\bm{a}_i\bm{a}_i^\top\right)$ is equal to the coefficient $t^r$ of the following determinant function
	\[\lim_{q\rightarrow \infty}\det\left(\frac{t}{q}\sum_{i\in \A_q\setminus S}\bm{a}_i\bm{a}_i^\top+\sum_{i\in S}\bm{a}_i\bm{a}_i^\top\right).\]
		
		Since $\lim_{q\rightarrow \infty}\frac{1}{q}\sum_{i\in \A_q\setminus S}\bm{a}_i\bm{a}_i^\top =\sum_{i\in [n]}\hat{x}_i\bm{a}_i\bm{a}_i^\top$ and $\det(\cdot)$ is a continuous function, thus we arrive at the conclusion.
	\end{enumerate}
\QEDA\end{proof}

Algorithm \ref{alg_derand_round_doe} proceeds as follows. We start with an empty subset $S$ of chosen elements, and for each $j\notin S$, we compute the expected $m$th power of objective function that $j$ will be chosen $H(S\cup{j})$. We update $S:=S\cup\{j^*\}$, where $j^*\in \arg\max_{j\in[n]\setminus S}H(S\cup{j})$. Then go to next iteration. This procedure will terminate if $|S|=k$. {Similar to Algorithm~\ref{alg_derand_round} and Algorithm~\ref{alg_derand_round_a}, the time complexity of Algorithm \ref{alg_derand_round_doe} is $O(m^4nk^2)$. Thus, in practice, we recommend Algorithm \ref{alg_rand_round_doe} for $D$-optimal design problem with repetitions.}

The approximation results for Algorithm~\ref{alg_derand_round_doe} are identical to those in Theorem~\ref{theorem:main3}, which is
\begin{theorem}\label{theorem:main2_det}
	For any positive integers $m\leq k\leq n$ and $\epsilon \in (0,1)$,
	\begin{enumerate}[(i)]
		\item deterministic \Cref{alg_derand_round_doe} is efficiently computable and yields $\frac{1}{e}$-approximation for the $D$-optimal design problem with repetitions; and
		\item given $\epsilon\in (0,1)$, if $k\geq \frac{m-1}{\epsilon}$, then deterministic \Cref{alg_derand_round_doe} yields $(1-\epsilon)$-approximation for the $D$-optimal design problem with repetitions.
	\end{enumerate}
\end{theorem}

\section{{Closing Remarks and }Conclusion}\label{sec_con}

{In this section, we make our final remarks about the proposed algorithms and present a conclusion of this paper.}

\noindent{\textbf{Closing Remarks:} We first remark that the proposed methods work also for A-optimality design, which has been studied in \cite{nikolov2018proportional}. In their paper, the authors also showed that the proposed methods might not work for other criteria. For $D$-optimal Design Problem without repetitions, if $k\approx m$, then we recommend sampling Algorithm~\ref{alg_rand_round} due to its efficiency and accuracy; if $k\gg m$, then we recommend sampling Algorithm~\ref{alg_rand_round_asymp} because of its efficiency and asymptotic optimality. For $D$-optimal Design Problem with repetitions, we recommend sampling Algorithm \ref{alg_rand_round_doe} since it is much more efficient than its deterministic counterpart.}

\noindent\textbf{Conclusion:} In this paper, we show that $D$-optimal design problem admits $
\frac{1}{e}$- approximation guarantee. That is, we propose a sampling algorithm and its deterministic implementation, whose solution is at most $\frac{1}{e}$ of the true optimal objective value, giving the first constant approximation ratio for this problem. We also analyze a different sampling algorithm, which achieves the asymptotic optimality, {i.e., the output of the algorithm is $(1-\epsilon)$-approximation if $k\geq \frac{4m}{\epsilon}+\frac{12}{\epsilon^2}\log(\frac{1}{\epsilon})$ for any $\epsilon \in (0,1)$.} For $D$-optimal design problem with repetitions, i.e., each experiment can be picked multiple times, our sampling algorithm and its derandomization improves asymptotic approximation ratio, {i.e., the output of the algorithm is $(1-\epsilon)$-approximation if $k\geq \frac{m-1}{\epsilon}$ for any $\epsilon \in (0,1)$.}
For future research, we would like to investigate if more sophisticated relaxation schemes can be used to improve the approximation analyses. Another direction is to prove the tightness of the approximation bounds. {In particular, we conjecture that for $D$-optimal design problem with or without repetitions, to achieve $(1-\epsilon)$-approximation, one must have $k=\Omega(m/\epsilon)$.}

\section*{Acknowledgment}

The first author has been supported by National Science Foundation grant CCF-1717947. Valuable
comments from the editors and two anonymous reviewers are gratefully acknowledged. The authors are also grateful to Jiaqi Wang for identifying a mistake in the derandomization algorithms.

\bibliography{Reference}

%
%
%
%
%

\newpage

\titleformat{\section}{\large\bfseries}{\appendixname~\thesection .}{0.5em}{}
\begin{appendices}
	
	\section{Proofs}\label{proofs}
	
	\subsection{Proof of \Cref{lem_key_lem}}\label{proof_lem_key_lem}
	\emkeylem*
	\begin{proof}Suppose that $T=\{i_1,\ldots,i_{|T|}\}$. Let matrix $A=[\bm{a}_{i_1},\ldots,\bm{a}_{i_{|T|}}]$, then
		\begin{align}
		\det\left(\sum_{i\in T}\bm{a}_i\bm{a}_i^\top\right)=\det\left(AA^\top\right).\label{eq_det_exp}
		\end{align}
		Next the right-hand side of \eqref{eq_det_exp} is equivalent to
		\begin{align*}
		\det\left(AA^\top\right)=\sum_{S\in {T\choose m}}\det\left(A_S\right)^2=\sum_{S\in {T\choose m}}\det\left(A_SA_S^{\top}\right)=\sum_{S\in {T\choose m}}\det\left(\sum_{i\in S}\bm{a}_i\bm{a}_i^\top\right),
		\end{align*}
		where $A_S$ is the submatrix of $A$ with columns from subset $S$, the first equality is due to Cauchy-Binet Formula \cite{broida1989comprehensive}, the second equality is because $A_S$ is a square matrix, and the last inequality is the definition of $A_SA_S^{\top}$.
	\QEDA\end{proof}

	\subsection{Proof of \Cref{lem_key_lem2}}\label{proof_lem_key_lem2}
	\emkeylemm*
	\begin{subequations}
		\begin{proof}[Proof of Lemma~\ref{lem_key_lem2}]Let $P=\diag(x) \in \Re^{n\times n}$ be the diagonal matrix with diagonal vector equal to $x$ and matrix $A=[\bm{a}_{1},\ldots,\bm{a}_{n}]$.
			By Lemma~\ref{lem_key_lem}, we have
			\begin{align}
			\det\left(\sum_{i\in [n]}x_i\bm{a}_i\bm{a}_i^\top\right)=\det\left(\sum_{i\in [n]}(\sqrt{x_i}\bm{a}_i)(\sqrt{x_i}\bm{a}_i)^\top\right)
			=\sum_{S\in {[n]\choose m}}\det\left(\sum_{i\in S}x_i\bm{a}_i\bm{a}_i^\top\right).\label{eq_indentity_za}
			\end{align}
			
			Note that $\sum_{i\in S}x_i\bm{a}_i\bm{a}_i^\top=A_SP_SA_S^{\top}$, where $A_S$ is the submatrix of $A$ with columns from subset $\S$, and $P_S$ is the square submatrix of $P$ with rows and columns from $S$. Thus, \eqref{eq_indentity_za} further yields
			\begin{align}
			\det\left(\sum_{i\in [n]}x_i\bm{a}_i\bm{a}_i^\top\right)&=\sum_{S\in {[n]\choose m}}\det\left(\sum_{i\in S}x_i\bm{a}_i\bm{a}_i^\top\right)
			=\sum_{S\in {[n]\choose m}}\det\left(A_SP_SA_S^{\top}\right) =\sum_{S\in {[n]\choose m}}\det\left(A_S\right)^2\det\left(P_S\right)\notag\\
			&=\sum_{S\in {[n]\choose m}}\prod_{i\in S}x_i
			\det\left(\sum_{i\in S}\bm{a}_i\bm{a}_i^\top\right)
			\end{align}
			where the third and fourth equalities are because the determinant of products of square matrices is equal to the products of individual determinants.
		\QEDA\end{proof}
	\end{subequations}
	
	\subsection{Proof of \Cref{thm_alg_rand_round}}\label{proof_thm_alg_rand_round}
	Before proving \Cref{thm_alg_rand_round}, we first introduce two well-known results for sum of homogeneous and symmetric polynomials.
	\begin{lemma}\label{lem_Maclaurin}(Maclaurin's inequality \cite{lin1993some}) Given a set $S$, an integer $s \in \{0,1,\cdots,|S|\}$ and nonnegative vector $x\in \Re_+^{|\S|}$, we must have
		\begin{align*}
		&\frac{1}{|S|}\left(\sum_{i\in S}x _i\right)\geq \sqrt[s]{\frac{1}{{|S|\choose s}}\left(\sum_{Q\in {S\choose s}}\prod_{i\in Q}x _i\right)}.
		\end{align*}
	\end{lemma}
	And
	\begin{lemma}\label{lem_Newton}(Generalized Newton's inequality \cite{xu2008generalized}) Given a set $S$, two nonnegative positive integers $s,\tau\in \Ze_+$ such that $s,\tau\leq |S|$ and nonnegative vector $x\in \Re_+^{|S|}$,
		then we have \begin{align*}
		&\frac{\left(\sum_{R\in {S\choose s}}\prod_{j\in R}x _j\right)}{{|S|\choose s}}\frac{\left(\sum_{R\in {S \choose{\tau}}}\prod_{i\in R}x _i\right)}{{|S|\choose \tau}}\geq \frac{\sum_{Q\in {S\choose s+\tau}}\prod_{i\in Q}x _i}{{|\S|\choose s+\tau}}
		\end{align*}
	\end{lemma}
	
	Now we are ready to prove the main proposition.
	\thmalgrandround*
	\begin{subequations}
		\begin{proof}
			According to Lemma~\ref{lemma:reduction} and sampling procedure in \eqref{eq:sampling1}, we have
			\begin{align}
			\Pr\left[T\subseteq\S\right]&=\prod_{j\in T}\hat{x}_j\frac{\sum_{R\in {[n]\setminus T\choose k-m}}\prod_{j\in R}\hat{x}_j}{\sum_{\bar{S}\in {[n]\choose k}}\prod_{i\in \bar{S}}\hat{x}_i}=\prod_{j\in T}\hat{x}_j\frac{\sum_{R\in {[n]\setminus T\choose k-m}}\prod_{j\in R}\hat{x}_j}{\sum_{\tau=0}^{m}\sum_{W\in {T\choose \tau}}\prod_{i\in W}\hat{x}_i\left(\sum_{Q\in {[n]\setminus T\choose k-\tau}}\prod_{i\in Q}\hat{x}_i\right)} \notag
			\end{align}
			where the second equality uses the following identity
			\[{[n]\choose k}=\bigcup_{\tau=0}^{m}\left\{W\cup Q: W\in {T\choose \tau}, Q\in {[n]\setminus T\choose k-\tau}\right\}.\]
			
			We now let
			\begin{align}
			A_T(x)=\frac{\sum_{\tau=0}^{m}\sum_{W\in {T\choose \tau}}\prod_{i\in W}\hat{x}_i\left(\sum_{Q\in {[n]\setminus T\choose k-\tau}}\prod_{i\in Q}\hat{x}_i\right)}{\sum_{R\in {[n]\setminus T\choose k-m}}\prod_{j\in R}\hat{x}_j}.\label{eq_def_A_T}
			\end{align}
			
			According to Definition~\ref{def:m-alpha}, it is sufficient to find a lower bound to $\frac{1}{\prod_{i\in T}\hat{x}_i}\Pr\left[T\subseteq \S\right]$, i.e.,
			\[\frac{1}{g(m,n,k)}\leq \min_{x}\left\{\frac{1}{\prod_{i\in T}\hat{x}_i}\Pr\left[T\subseteq \S\right]=\frac{1}{A_T(x)}:\sum_{i\in [n]}\hat{x}_i=k, x \in [0,1]^n\right\}.\]

			Or equivalently, we would like to find an upper bound of $A_{T}(x)$ for any $x$ which satisfies $\sum_{i\in [n]}\hat{x}_i=k, x \in [0,1]^n$, i.e., show that
			\begin{align}
			g(m,n,k)\geq\max_{x}\left\{A_T(x):\sum_{i\in [n]}\hat{x}_i=k, x \in [0,1]^n\right\}.\label{eq_def_A_T_max}
			\end{align}
			
			In the following steps, we first observe that in \eqref{eq_def_A_T}, the components of $\{x_i\}_{i\in T}$ and $\{x_i\}_{i\in [n]\setminus T}$ are both symmetric in the expression of $A_T(x)$. We will show that for the worst case, $\{x_i\}_{i\in T}$ are all equal and $\{x_i\}_{i\in [n]\setminus T}$ are also equal. We also show that $\hat{x}_j \leq \hat{x}_i$ for each $i\in T$ and $j\in [n]\setminus T$. This allows us to reduce the optimization problem in R.H.S. of \eqref{eq_def_A_T_max} to a single variable optimization problem, i.e., \eqref{eq_def_G}. The proof is now separated into following three claims.
			
			\begin{enumerate}[(i)]
				
				\item First, we claim that
				\begin{claim}\label{claim1}The optimal solution to \eqref{eq_def_A_T_max} must satisfy the following condition - for each $i\in T$ and $j\in [n]\setminus T$, $\hat{x}_j \leq \hat{x}_i$.
				\end{claim}
				\begin{proof}We prove it by contradiction. Suppose that there exists $i'\in T$ and $j'\in [n]\setminus T$, where $\hat{x}_{i'}<\hat{x}_{j'}$.
					By collecting the coefficients of $1,\hat{x}_{i'},\hat{x}_{j'},\hat{x}_{i'}\hat{x}_{j'}$, we have
					\begin{align*}
					&A_{T}(x)=\frac{b_1+b_2\hat{x}_{i'}+b_2\hat{x}_{j'}+b_3\hat{x}_{i'}\hat{x}_{j'}}{c_1+c_2\hat{x}_{j'}}
					\end{align*}
					where $b_1,b_2,b_3,c_1,c_2$ are all non-negative numbers with
					\begin{align*}
					&b_1=\sum_{\bar{S}\in {[n]\setminus\{i',j'\}\choose k}}\prod_{i\in \bar{S}}\hat{x}_i,b_2=\sum_{\bar{S}\in {[n]\setminus\{i',j'\}\choose k-1}}\prod_{i\in \bar{S}}\hat{x}_i,b_3=\sum_{\bar{S}\in {[n]\setminus\{i',j'\}\choose k-2}}\prod_{i\in \bar{S}}\hat{x}_i\\
					&c_1=\sum_{R\in {[n]\setminus (T\cup\{j'\})\choose k-m}}\prod_{j\in R}\hat{x}_j, c_2=\sum_{R\in {[n]\setminus (T\cup\{j'\})\choose k-m-1}}\prod_{j\in R}\hat{x}_j.
					\end{align*}
					
					Note that
					\[\hat{x}_{i'}\hat{x}_{j'}\leq \frac{1}{4}\left(\hat{x}_{i'}+\hat{x}_{j'}\right)^2.\]
					Therefore, $A_{T}(x)$ has a larger value if we replace $\hat{x}_{i'},\hat{x}_{j'}$ by their average, i.e. $\hat{x}_{i'}:=\frac{1}{2}(\hat{x}_{i'}+\hat{x}_{j'}),\hat{x}_{j'}:=\frac{1}{2}(\hat{x}_{i'}+\hat{x}_{j'})$.
				\QEDB\end{proof}

				\item Next, we claim that
				\begin{claim}\label{claim2}for any feasible $x$ to \eqref{eq_def_A_T_max}, and for each $S \subseteq [n]$ and $s \in \{0,1,\cdots,|S|\}$, we must have
					\[\sum_{Q\in {S\choose s}}\prod_{i\in Q}\hat{x}_i \leq \frac{{|S|\choose s}}{|S|^{s}}\left(\sum_{i\in S}\hat{x}_i\right)^{s}.\]
				\end{claim}
				\begin{proof}This directly follows from Lemma~\ref{lem_Maclaurin}.
				\QEDB\end{proof}
				And also
				\begin{claim}\label{claim3}for each $T\subseteq [n]$ with $|T|=k$ and $\tau \in \{0,1,\cdots,m\}$,
					\[\sum_{Q\in {[n]\setminus T\choose k-\tau}}\prod_{i\in Q}\hat{x}_i \leq \frac{{n-m\choose k-\tau}}{(n-m)^{m-\tau}{n-m\choose k-m}}\left(\sum_{R\in {[n]\setminus T\choose k-m}}\prod_{j\in R}\hat{x}_j\right)\left(\sum_{i\in [n]\setminus T}\hat{x}_i\right)^{m-\tau}.\]
				\end{claim}
				\begin{proof}This can be shown by Claim \ref{claim2} and Lemma~\ref{lem_Newton}
					\begin{align*}
					&\frac{{n-m\choose k-\tau}}{(n-m)^{m-\tau}{n-m\choose k-m}}\left(\sum_{R\in {[n]\setminus T\choose k-m}}\prod_{j\in R}\hat{x}_j\right)\left(\sum_{i\in [n]\setminus T}\hat{x}_i\right)^{m-\tau}
					\\
					&\geq\frac{{n-m\choose k-\tau}}{{n-m\choose m-\tau}{n-m\choose k-m}}\left(\sum_{R\in {[n]\setminus T\choose k-m}}\prod_{j\in R}\hat{x}_j\right)\left(\sum_{S\in {[n]\setminus T \choose{m-\tau}}}\prod_{i\in S}\hat{x}_i\right)\geq \sum_{Q\in {[n]\setminus T\choose k-\tau}}\prod_{i\in Q}\hat{x}_i
					\end{align*}
					where the first inequality is due to Claim \ref{claim2}, and the last inequality is because of Lemma~\ref{lem_Newton}.
				\QEDB\end{proof}
				
				\item
				Thus, by Claim \ref{claim3}, for any feasible $x$ to \eqref{eq_def_A_T_max}, $A_{T}(x)$ in \eqref{eq_def_A_T} can be upper bounded as
				\begin{align}
				A_{T}(x)&\leq \sum_{\tau=0}^{m}\frac{{n-m\choose k-\tau}}{(n-m)^{m-\tau}{n-m\choose k-m}}\left(\sum_{i\in [n]\setminus T}\hat{x}_i\right)^{m-\tau}\sum_{W\in {T\choose \tau}}\prod_{i\in W}\hat{x}_i\notag \\
				&\leq \sum_{\tau=0}^{m}\frac{{n-m\choose k-\tau}}{(n-m)^{m-\tau}{n-m\choose k-m}}\frac{{m\choose\tau}}{m^{\tau}}\left(\sum_{i\in [n]\setminus T}\hat{x}_i\right)^{m-\tau}\left(\sum_{i\in T}\hat{x}_i\right)^{\tau}\notag\\
				&\leq \max_{y}\left\{\sum_{\tau=0}^{m}\frac{{n-m\choose k-\tau}}{(n-m)^{m-\tau}{n-m\choose k-m}}\frac{{m\choose\tau}}{m^{\tau}}\left(k-y\right)^{m-\tau}\left(y\right)^{\tau}:\frac{mk}{n}\leq y\leq m\right\}:=g(m,n,k)
				\label{eq_ub_A_T}
				\end{align}
				where the second inequality is due to Claim \ref{claim2}, and the last inequality is because we let $y=\sum_{i\in T}\hat{x}_i$ which is no larger than $m$, maximize over it and Claim \ref{claim1} yields that $y/m\geq (k-y)/(n-m)$, i.e. $\frac{mk}{n}\leq y\leq m$. This completes the proof.
			\end{enumerate}
		\QEDA\end{proof}
	\end{subequations}
	
	\subsection{Proof of \Cref{thm_alg_rand_round_bnd}}\label{proof_thm_alg_rand_round_bnd}
	\thmalgrandroundbnd*
	
	\begin{proof}
		\begin{subequations}
			\begin{enumerate}[(i)]
				\item First of all, we prove the following claim.
				\begin{claim}\label{claim_G} For any $m\leq k\leq n$, we have\[g(m,n,k) \leq g(m,n+1,k).\]
				\end{claim}
				\begin{proof}
					Let $y^*$ be the maximizer to \eqref{eq_def_G} for any given $m\leq k\leq n$, i.e.,
					\[g(m,n,k)=\sum_{\tau=0}^{m}\frac{{n-m\choose k-\tau}}{(n-m)^{m-\tau}{n-m\choose k-m}}\frac{{m\choose\tau}}{m^{\tau}}\left(k-y^*\right)^{m-\tau}\left(y^*\right)^{\tau}.\]
					
					Clearly, $y^*$ is feasible to \eqref{eq_def_G} with pair $(m,n+1,k)$. We only need to show that
					\[g(m,n,k)\leq \sum_{\tau=0}^{m}\frac{{n+1-m\choose k-\tau}}{(n+1-m)^{m-\tau}{n+1-m\choose k-m}}\frac{{m\choose\tau}}{m^{\tau}}\left(k-y^*\right)^{m-\tau}\left(y^*\right)^{\tau}.\]
					In other words, it is sufficient to show for any $0\leq \tau \leq m$,
					\[\frac{{n-m\choose k-\tau}}{(n-m)^{m-\tau}{n-m\choose k-m}}\leq \frac{{n+1-m\choose k-\tau}}{(n+1-m)^{m-\tau}{n+1-m\choose k-m}},\]
					which is equivalent to prove
					\[\frac{n-k}{n-m}\cdot \frac{n-k-1}{n-m}\cdots \frac{n-k-m+\tau+1}{n-m}\leq \frac{n+1-k}{n+1-m}\cdot \frac{n+1-k-1}{n+1-m}\cdots \frac{n+1-k-m+\tau+1}{n+1-m}.\]
					The above inequality holds since for any positive integers $p,q$ with $p<q$, we must have $\frac{p}{q}\leq \frac{p+1}{q+1}$.
				\QEDB\end{proof}
				

				\item By Claim~\ref{claim_G}, it is sufficient to investigate the bound $\lim_{n'\rightarrow \infty}g(m,n',k)$, which provides an upper bound to $g(m,n,k)$ for any integers $n\geq k\geq m$. Therefore, from now on, we only consider the case when $n\rightarrow \infty$ for any fixed $k\geq m$.
				
				Note that for any given $y$, $\sum_{\tau=0}^{m}\frac{{n-m\choose k-\tau}}{(n-m)^{m-\tau}{n-m\choose k-m}}\frac{{m\choose\tau}}{m^{\tau}}(k-y)^{m-\tau}y^{\tau}$ is the coefficient of $t^k$ in the following polynomial:
				\begin{align*}
				R_1(t):=\frac{(n-m)^{k-m}}{(k-y)^{k-m}{n-m\choose k-m}}\left(1+\frac{k-y}{n-m}t\right)^{n-m}\left(1+\frac{y}{m}t\right)^m
				\end{align*}
				which is upper bounded by
				\begin{align*}
				R_2(t)&:=\frac{(n-m)^{k-m}}{(k-y)^{k-m}{n-m\choose k-m}}\left(1+\frac{k-y}{n-m}t+\frac{1}{2!}\left(\frac{k-y}{n-m}t\right)^2+\ldots\right)^{n-m}\left(1+\frac{y}{m}t+\frac{1}{2!}\left(\frac{y}{m}t\right)^2+\ldots\right)^m\\
				&=\frac{(n-m)^{k-m}}{(k-y)^{k-m}{n-m\choose k-m}}\left(e^{\frac{k-y}{n-m}t}\right)^{n-m}\left(e^{\frac{y}{m}t}\right)^m
				\end{align*}
				because of the inequality $e^{r}= 1+r+\frac{1}{2}r^2+\ldots$ for any $r$ and $t\geq 0$. Therefore, we also have 
				\begin{align*}
				&\lim_{n\rightarrow \infty}\frac{1}{k!}\frac{d^kR_1(t)}{dt^k}\bigg\vert_{t=0}=\lim_{n\rightarrow \infty}\sum_{\tau=0}^{m}\frac{{n-m\choose k-\tau}}{(n-m)^{m-\tau}{n-m\choose k-m}}\frac{{m\choose\tau}}{m^{\tau}}(k-y)^{m-\tau}y^{\tau} \\
				&\leq \lim_{n\rightarrow \infty}\sum_{\tau=0}^{m}\frac{{n-m\choose k-\tau}}{(n-m)^{m-\tau}{n-m\choose k-m}}\frac{{m\choose\tau}}{m^{\tau}}(k-y)^{m-\tau}y^{\tau}\\
				&+\sum_{\tau=0}^{m}\sum_{\begin{subarray}{c}
					i_j\in\Ze_+, \forall j\in [n]\\
					\sum_{j\in [n-m]}i_j=k-\tau\\
					\sum_{j\in [n]\setminus [n-m]}i_j=\tau\\
					\max_{j\in [n]}i_j\geq 2
					\end{subarray}}\frac{1}{\prod_{j\in [n]}i_j!}\frac{1}{(n-m)^{m-\tau}{n-m\choose k-m}}\frac{1}{m^{\tau}}(k-y)^{m-\tau}y^{\tau}:=\lim_{n\rightarrow \infty}\frac{1}{k!}\frac{d^kR_2(t)}{dt^k}\bigg\vert_{t=0}\\
				&= \lim_{n\rightarrow \infty}\frac{k^k}{k!}\frac{(n-m)^{k-m}}{(k-y)^{k-m}{n-m\choose k-m}}\leq \lim_{n\rightarrow \infty}\frac{k^k}{k!}\frac{(n-m)^{k-m}}{(k-m)^{k-m}{n-m\choose k-m}}\\
				&= \frac{k^k}{k!}\frac{(k-m)!}{(k-m)^{k-m}}:=R_3(m,k)
				\end{align*}
				where the first inequality is due to the non-negativity of the second term of $\frac{1}{k!}\frac{d^kR_2(t)}{dt^k}\bigg\vert_{t=0}$, the second and third equalities are because of two equivalent definitions of $R_2(t)$, the last inequality is due to $y\leq m$ and the fourth equality holds because of $n\rightarrow \infty$.

				Note that $R_3(m,k)$ is nondecreasing over $k\in [m,\infty)$. Indeed, for any given $m$,
				\begin{align*}
				&\log\frac{R_3(m,k+1)}{R_3(m,k)}=k\log\left(1+\frac{1}{k}\right)-(k-m)\log\left(1+\frac{1}{(k-m)}\right),
				\end{align*}
				whose first derivative over $k$ is equal to
				\begin{align*}
				&\log\left(1+\frac{1}{k}\right)-\frac{1}{k+1}-\log\left(1+\frac{1}{(k-m)}\right)+\frac{1}{k-m+1} \leq 0,
				\end{align*}
				i.e., $\log\frac{R_3(m,k+1)}{R_3(m,k)}$ is nonincreasing over $k$.
				Therefore,
				\begin{align*}
				&\log\frac{R_3(m,k+1)}{R_3(m,k)}=k\log\left(1+\frac{1}{k}\right)-(k-m)\log\left(1+\frac{1}{(k-m)}\right)\geq \lim_{k\rightarrow \infty}\log\frac{R_3(m,k+1)}{R_3(m,k)}=0
				\end{align*}
				
				Thus, $R_3(m,k)$ is upper bounded when $k\rightarrow\infty$, i.e.,
				\begin{align*}
				&R_3(m,k) \leq \lim_{k'\rightarrow \infty}R_3(m,k')=\lim_{k'\rightarrow \infty}\left[\left(1-\frac{m}{k'}\right)^{-\frac{k'}{m}}\right]^m\frac{(k'-m)^m}{k'(k'-1)\cdots (k'-m+1)}= e^m,
				\end{align*}
				where the last equality is due to the fact that $\lim_{k'\rightarrow \infty}\left(1-\frac{m}{k'}\right)^{-\frac{k'}{m}}=e$ and $\lim_{k'\rightarrow \infty}\frac{(k'-m)^m}{k'(k'-1)\cdots (k'-m+1)}=1$. Therefore,
				\begin{align*}
				\lim_{n\rightarrow \infty}[g(m,n,k)]^{\frac{1}{m}}&=\lim_{n\rightarrow \infty}\left[\max_{y}\left\{\sum_{\tau=0}^{m}\frac{{n-m\choose k-\tau}}{(n-m)^{m-\tau}{n-m\choose k-m}}\frac{{m\choose\tau}}{m^{\tau}}(k-y)^{m-\tau}y^{\tau}:\frac{mk}{n}\leq y\leq m\right\}\right]^{\frac{1}{m}}\\
				&\leq\left[R_3(m,k)\right]^{\frac{1}{m}}\leq e
				\end{align*}
				
				\item We now compute another bound $1+\frac{k}{k-m+1}$ for $[g(m,n,k)]^{\frac{1}{m}}$, which can be smaller than $e$ when $k$ is large. By Claim~\ref{claim_G}, we have
				\[g(m,n,k)\leq \lim_{n'\rightarrow \infty} g(m,n',k)= 
				\max_{y}\left\{\sum_{\tau=0}^{m}\frac{(k-m)!}{(k-\tau)!}\frac{{m\choose\tau}}{m^{\tau}}(k-y)^{m-\tau}y^{\tau}:0\leq y\leq m\right\}. \]
				
				Note that $0\leq y\leq m$, thus $k-y \leq k$. Therefore, we have
				\begin{align*}
				&\lim_{n\rightarrow \infty}g(m,n,k) 
				\leq
				\sum_{\tau=0}^{m}\frac{(k-m)!{m\choose\tau}}{(k-\tau)!}k^{m-\tau}\leq \sum_{\tau=0}^{m}\left(\frac{k}{k-m+1}\right)^{m-\tau}=\left(1+\frac{k}{k-m+1}\right)^{m},
				\end{align*}
				where the last inequality is due to $\frac{(k-m)!}{(k-\tau)!} =\frac{1}{(k-\tau)\cdots (k-m+1)} \leq \left(\frac{1}{k-m+1}\right)^{m-\tau}$.
				
				Therefore, we have
				\begin{align*}
				[g(m,n,k)]^{\frac{1}{m}}&=\left[\max_{y}\left\{\sum_{\tau=0}^{m}\frac{{n-m\choose k-\tau}}{(n-m)^{m-\tau}{n-m\choose k-m}}\frac{{m\choose\tau}}{m^{\tau}}(k-y)^{m-\tau}y^{\tau}:\frac{mk}{n}\leq y\leq m\right\}\right]^{\frac{1}{m}}\\
				&\leq 1+\frac{k}{k-m+1}
				\end{align*}
				for any $m\leq k\leq n$.

			\end{enumerate}
		\end{subequations}
	\QEDA\end{proof}
	
\end{appendices}

\end{document}